\pdfoutput=1
\documentclass[letterpaper]{article} 
\usepackage{aaai25}  
\usepackage{times}  
\usepackage{helvet}  
\usepackage{courier}  
\usepackage[hyphens]{url}  
\usepackage{graphicx} 
\usepackage{natbib}  
\usepackage{caption} 
\usepackage{algorithm}
\usepackage{algorithmic}
\usepackage{newfloat}
\usepackage{listings}
\DeclareCaptionStyle{ruled}{labelfont=normalfont,labelsep=colon,strut=off} 
\floatstyle{ruled}
\newfloat{listing}{tb}{lst}{}
\floatname{listing}{Listing}
\usepackage{amsthm}
\usepackage{amsfonts}
\usepackage{amssymb}
\usepackage{amsmath}
\usepackage{enumitem}
\usepackage{tcolorbox}
\tcbuselibrary{listingsutf8}
\tcbuselibrary{skins,breakable}
\usepackage{varwidth}
\usepackage{xcolor}

\definecolor{myboxbackground}{RGB}{255,215,0}
\definecolor{backgroundblue}{RGB}{155, 221, 255}
\definecolor{myboxborder}{RGB}{117, 170, 219}

\definecolor{myboxtitle}{rgb}{0.95, 0.2, 0.6}

\tcbset{enhanced,skin=enhancedlast jigsaw, attach boxed title to top left={xshift=-4mm,yshift=-0.5mm}, width = 0.95\linewidth, center, fonttitle=\bfseries\sffamily,varwidth boxed title=0.7\linewidth, colbacktitle=white,colframe=myboxborder,
   interior style={top color=myboxborder!10!white,bottom color=backgroundblue!30!white},
   boxed title style={empty,arc=0pt,outer arc=0pt,boxrule=0pt},
   underlay boxed title={
\fill[myboxborder] (title.north west) -- (title.north east) -- +(\tcboxedtitleheight-1mm,-\tcboxedtitleheight+1mm)
-- ([xshift=4mm,yshift=0.5mm]frame.north east) -- +(0mm,-1mm) -- (title.south west) -- cycle;
     \fill[myboxborder!75!white!65!black] ([yshift=-0.5mm]frame.north west)
       -- +(-0.4,0) -- +(0,-0.3) -- cycle;
     \fill[myboxborder!75!white!65!black] ([yshift=-0.5mm]frame.north east)
       -- +(0,-0.3) -- +(0.4,0) -- cycle;  }}

\newtheorem{theorem}{Theorem}
\newtheorem{lemma}{Lemma}
\newtheorem{proposition}{Proposition}
\newtheorem{corollary}{Corollary}

\newtheorem{definition}{Definition}

\title{Paid with Models: Optimal Contract Design for Collaborative Machine Learning}
\author {
    Bingchen Wang\textsuperscript{\rm 1}\thanks{Corresponding Author.},
    Zhaoxuan Wu\textsuperscript{\rm 1,2},
    Fusheng Liu\textsuperscript{\rm 1},
    Bryan Kian Hsiang Low\textsuperscript{\rm 3}
}
\affiliations {
    \textsuperscript{\rm 1}Institute of Data Science, National University of Singapore\\
    \textsuperscript{\rm 2}Singapore-MIT Alliance for Research and Technology, Republic of Singapore\\
    \textsuperscript{\rm 3}Department of Computer Science, National University of Singapore\\
    bingchen@nus.edu.sg, \{wu.zhaoxuan, fusheng\}@u.nus.edu, lowkh@comp.nus.edu.sg
}

\begin{document}

\maketitle

\begin{abstract}
Collaborative machine learning (CML) provides a promising paradigm for democratizing advanced technologies by enabling cost-sharing among participants. However, the potential for rent-seeking behaviors among parties can undermine such collaborations. Contract theory presents a viable solution by rewarding participants with models of varying accuracy based on their contributions. However, unlike monetary compensation, using models as rewards introduces unique challenges, particularly due to the stochastic nature of these rewards when contribution costs are privately held information. This paper formalizes the optimal contracting problem within CML and proposes a transformation that simplifies the non-convex optimization problem into one that can be solved through convex optimization algorithms. We conduct a detailed analysis of the properties that an optimal contract must satisfy when models serve as the rewards, and we explore the potential benefits and welfare implications of these contract-driven CML schemes through numerical experiments.
\end{abstract}

%

\section{Introduction}
Training a state-of-the-art machine learning (ML) model is a Herculean task due to the requirement of an  enormous amount of data and computational resources. The exorbitant cost often precludes budget-constrained small parties from training a model on their own, resulting in a high industrial concentration where top-performing models are owned by big firms \cite{aiindex2024}. In this regard, collaborative machine learning (CML) provides a promising crowdsourcing paradigm. The advent of CML schemes like federated learning \cite{mcmahan_communication-efficient_2017, kairouz_advances_2021, sheller_federated_2020, nguyen_deep_2022} allows participants to join their resources for model training and share the training cost that would otherwise be insurmountable at an individual level. Despite their great potential, such schemes might not make economic sense. As is shown by \citet{karimireddy_mechanisms_2022}, catastrophic freeriding can occur when profit-maximizing parties in a collaboration have the ability to observe each other's data collection costs. This issue can be mitigated through the role of a scheme coordinator who conducts model training on the parties' behalf and rewards models with modified accuracy levels based on the parties' contributions. Practically, this could be achieved through the \textit{design of contracts}, where the scheme coordinator acts as the \textit{principal} and each participating party of the scheme acts as the \textit{agent}.

Prior to our work, there has been a line of research that resorts to contract theory to address the incentive issue in collaborative machine learning \citep{kang_incentive_2019, ding_incentive_2020, karimireddy_mechanisms_2022, Liu2023}, but most of them focus on using money as the reward for the collaboration. \citet{karimireddy_mechanisms_2022} attends to the administration of models with different accuracy levels as rewards, while their primary focus is on the case where the scheme coordinator can directly observe each party's data collection costs. However, in reality, the cost of contribution is typically \textit{private information} known only to the contributing party. For instance, consider a CML scheme where private computing firms pull together their GPUs for the training of a language model for code generation. Each firm could face a different vendor price and incur dissimilar maintenance cost of the chips. As another example, consider the CML scheme where investment firms join their privately curated data for the training of an investment model. To gather the data, each firm needs to recruit analysts, the overheads of which are usually determined by conditions of the local labor market and the firm's own incentive policies. The differences in the operating environments cause the parties of a CML scheme to have a wide range of per-unit contribution costs. While the scheme coordinator can be an expert in the domain field, thereby possessing some general information about the process, it remains challenging for them to gauge the exact costs borne by the parties. Even if the parties willingly inform the coordinator of their costs, the coordinator cannot verify the truthfulness of these reports without incurring significant auditing expenses. Worse still, a rent-seeking party may cheat by misreporting their cost if it leads to higher profits being gained from the scheme. This information asymmetry results in what is known as a \textit{principal-agency problem} in economic literature (see \citealt{mwg1995, laffont_theory_2002, bolton2004} for a comprehensive treatment of the subject).

In the presence of private information, optimal contract design with models as the rewards poses unique challenges that distinguish it from its economic counterparts. For one, unlike money, models are a non-rivalrous and non-exclusive good, and can be replicated and offered to the participants at a nominal cost if not free of charge. Therefore, the scheme coordinator would find it tempting to offer less capable parties a good-performing model as long as it does not cause the more capable parties to cheat. For another, the administrable model rewards are constrained by the accuracy level of the model trained using all parties' data or computational resources. Due to \textit{incomplete information}, the coordinator cannot observe the exact numbers of parties with different contribution costs in the CML scheme, and consequently cannot determine the exact accuracy level of the collectively trained model before the training completes. This makes the rewards of the contract stochastic ex-ante. The optimal contracting problem for CML needs to accommodate these challenges, whilst heeding the classical requirements of individual rationality and incentive compatibility. To this end, our paper makes the following contributions:
\begin{itemize}
    \item We provide a \textit{coherent formalization} of the optimal contracting problem in CML with models as the rewards, casting it as a constrained optimization problem.
    \item We \textit{simplify} the original non-convex constrained optimization problem into one that can be solved using numerical optimization algorithms.
    \item We conduct \textit{theoretical analysis of the constraints}, delineating the properties optimally designed contracts should obey for both scenarios when the coordinator can and cannot observe parties' contribution costs.
    \item We \textit{illustrate the potential and the welfare implications} of optimally designed contracts through numerical experiments, showing that, inter alia, it could help small parties surmount the cost barrier of model training and reaping the reward of emergent technologies.
\end{itemize}

\section{Problem Setup}
\begin{figure*}[!htb]
    \centering
    \includegraphics[width=0.85\textwidth]{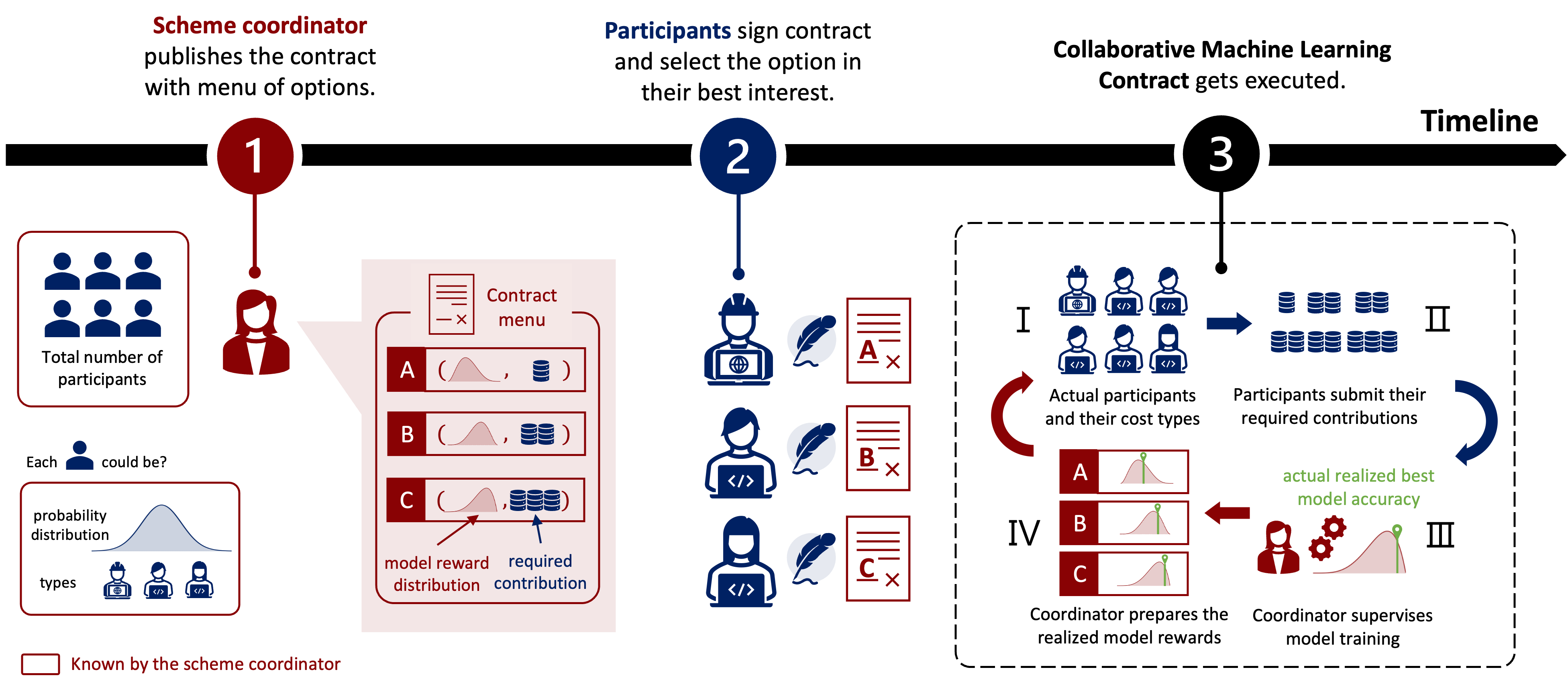}
    \caption{Optimal Contract Design for Collaborative Machine Learning: The Timeline.}
    \label{fig:timeline}
\end{figure*}
\subsubsection*{Collaborative Machine Learning.} We consider a typical setting of collaborative machine learning where budget-constrained parties contribute their resources, such as data and GPUs, to collectively train a model through the mediation of a scheme coordinator. We capture the dependency between contributed resources and model accuracy through the function $a(\cdot)$, and make the standard assumptions that more contribution leads to better model performance, $a^\prime(\cdot)>0$, and that the same amount of contribution has diminishing marginal returns as the total contribution increases, $a^{\prime\prime}(\cdot)<0$. To ensure that the framework is well-grounded economically, we additionally represent the dependency between accuracy of a model and the economic profit gained by the participant through a weakly concave and increasing valuation function $v(\cdot)$, with $v^\prime(\cdot)>0$ and $v^{\prime\prime}(\cdot)\leq 0$. In the computing firm example mentioned in the introduction, the accuracy function captures the fact that the use of more GPUs enables more iterations of model training to be undertaken within the given time span; the valuation function reflects the fact that improved performance of the language model reduces human overheads—enabling a firm to hire fewer software engineers for the same amount of work. In the investment firm example, $a(\cdot)$ captures the fact that more data leads to a fuller picture of the market landscape, which in turn boosts the predictive accuracy of the trained model; $v(\cdot)$ embodies the fact that a model's predictive accuracy directly affects the quality of a firm's investment decisions and better decisions lead to more revenues being earned by the firm. To consolidate the intuition, we henceforth use data consistently throughout the paper as the resources contributed by parties in the collaboration, whereas the analysis can also apply readily to other resource types, like GPUs or computation devices. 

\subsubsection*{Principal-Agent Problem.} There are two distinct features of the CML setting mentioned above: Firstly, there is a \textit{conflict of interests} between the scheme coordinator, who aims to maximize the performance of the collectively trained model by encouraging data contribution, and the scheme participant, who hopes to receive a good-performing model by contributing as little data as possible since collection is costly. Secondly, the scheme participant possesses \textit{information advantage} over the scheme coordinator, as the per-unit data collection cost is directly known by the party themself but may not be observed by the coordinator likely due to costly verification or auditing process. This opens up the possibility that a party might cheat by misreporting their cost so that they can be compensated better by the coordinator.
To facilitate the formalization of the model, we introduce the following notation: Let $I$ denote the total number of distinct types of private per-unit data collection costs in the population—henceforth called \textit{private types} or \textit{types}; we denote the per-unit data costs by $c_i, i = 1,\dots, I$,  and order them decreasingly, $c_1 > \ldots > c_I$, for analytical tractability. Hence, a type-$1$ party incurs the highest per-unit cost while contributing data to the CML scheme. We let $n = (n_1, \ldots, n_I)$ be a vector that counts the numbers of parties of each type in the collaboration and $N= \sum_i^I n_i$ denote the total number of participants of the CML scheme. Here, we assume the coordinator observes $N$ but not $n$. She nevertheless has the general knowledge that $n$ follows a multinomial distribution $\mathrm{Mul}(N,p)$ with probabilities $p = (p_1, \ldots, p_I)$   for the types. We denote the model reward received by a type-$i$ party by $r_i$ measured in model accuracy, their data contribution by $m_i$ measured in amount, and their reservation utility by $f_i$ measured in monetary terms and defined as the highest profits they could achieve by not participating in the CML scheme. 

\subsubsection*{Model as the Reward.} When we use models as the rewards for the CML scheme, a natural \textit{budget constraint} we need to abide by is
\begin{equation*}
    \Vert r(n)\Vert_\infty \leq  a\left(\sum_{i=1}^I n_i m_i\right), \forall n \in \mathrm{Mul}(N,p)\ ,
\end{equation*}
where $r(n) = \big(r_1(n), \ldots, r_I(n)\big)$ is a vector with elements specifying the model rewards received by different types of parties under the realized combination $n$. In words, for every possible combination of types, the maximal model accuracy reward assigned to the parties cannot surpass the accuracy of the collectively trained model. Fixing the contributions $m_i, i=1,\dots, I$ from each type of parties, we see that this upper bound is dependent on $n$. Consequently, the model rewards the coordinator can assign also depend on $n$. At the stage of contract design, unless the coordinator can perfectly observe the private types of the parties, she does not know $n$ and therefore the model rewards are stochastic. We accommodate this fact with the notation $\mathbf{r}_i,\forall i = 1,\dots, I$, which embodies the reward distributions for the types. Figure \ref{fig:timeline} shows the complete timeline of the contracting process.

\section{Optimal Contracting in CML}
With the notation in place, we are now ready to formalize the optimal contracting problem in CML.
\subsubsection*{Coordinator's Objective.}
The objective of the coordinator reflects the goal of the CML scheme. Here we assume that she aims to maximize the expected model accuracy,
\begin{equation*}
\mathbb{E}_{n\sim\mathrm{Mul}(N,p)} \left[ a\left(\sum_{i=1}^I n_i m_i\right)\right].
\end{equation*}
For readability, we henceforth abbreviate $\mathbb{E}_{n\sim\mathrm{Mul}(N,p)}$ to $\mathbb{E}$ unless otherwise stated.

\subsubsection*{Participant's Utility Function.}
Each participant is assumed to be a von Neumann–Morgenstern utility maximizer \citep{Neumann_Morgenstern_1944}, who has the following utility function:
\begin{equation}
\label{participants' utility functions}
    u_i(\mathbf{r}_i,m_i) = \mathbb{E}_{n_i \geq 1}[v(r_i)] - c_i m_i\ .
\end{equation}
Recall that $r_i$ denotes the stochastic model reward (measured in accuracy), which takes on different values for different $n$; $m_i$ is the amount of data contribution; and $c_i$ denotes the private per-unit data cost. $v(\cdot)$ reflects the profit mechanism that maps a model accuracy level to a pecuniary amount. The expectation operator $\mathbb{E}_{n_i \geq 1}$ codifies the information advantage possessed by the type-$i$ participant—when they decide to participate, they know in advance that there is at least one party making the type-$i$ commitment. The aim of each participant is to maximize their expected net profit, represented by the utility function in (\ref{participants' utility functions}).

\subsubsection*{Individual Rationality.}
To ensure the contract is well designed, it must pass the first test that it gives parties of different types enough incentives to join the CML scheme. Formally, this requires that a party upon choosing the contract option designed for their type cannot be made worse off than them not participating in the CML scheme. This is known as the \textit{individual rationality} (IR) conditions. To formalize the idea, we should specify the reservation utility (a.k.a. opportunity cost) for each type of parties. Here we define a party's reservation utility to be the utility level they achieves by training a model on their own, which amounts to solving the following optimization problem:
\begin{equation}
    f_i \triangleq \max_{m\geq 0} \; \tilde{u}_i(m) = v(a(m)) - c_i m\ . \label{indiv training problem}
\end{equation}
We make the following observation upon solving this problem, the proof of which can be found in Appendix A.
\begin{proposition}
    Let $\bar{m}_i$ denote the data contribution a type-$i$ party is willing to commit to when training a model on their own. If $c_i \leq c_j$, then $\bar{m}_i \geq \bar{m}_j$, and $f_i \geq f_j\ .$
\end{proposition}
The proposition states that when training a model alone, a party with lower data cost is willing to utilize more data and will end up with a better model and generate higher profits. We can write the IR conditions as follows:
\begin{equation*}
    \mathbb{E}_{n_i \geq 1}[v(r_i)] - c_i m_i \geq f_i\ , \forall i\ .
\end{equation*}
Using this formulation, readers familiar with the economic literature could also interpret $f_i$ as the fixed cost and $c_i$ the variable cost of joining the CML scheme for type-$i$ parties. 
\subsubsection*{Incentive Compatibility.}
Designing a contract $\mathcal{C}$ in the presence of hidden information seems highly complicated, as participants could lie about their private costs. Luckily, with the aid of the revelation principle (\citealp{myerson_optimal_1981, myerson_optimal_1982}; see \citealp{mwg1995} for a lucid explanation on the concept), we could confine the design space to one contract option per private type, with each party choosing the option designed for their type upon signing the contract. In the CML setting, this translates to designing a contract option that stipulates a required data contribution $m_i$ and the corresponding reward distribution $\mathbf{r}_i$ for each type of parties. The veracity of parties is formally ensured through the \textit{incentive compatibility} (IC) constraints:
\begin{equation*}
    \mathbb{E}_{n_i \geq 1}[v(r_i)] - c_i m_i \geq \mathbb{E}_{n_j \geq 1}[v(r_j)] - c_i m_j, \forall i,j\ .
\end{equation*}
In words, for each party, choosing the option designed for their type yields a weakly higher expected profit than choosing any other option on offer. We assume a tie is broken in favour of veracity. Note the change from $n_i\geq 1$ to $n_j\geq 1$ in the expectation operator in defining the IC constraints, which again reflects participant's information advantage—they know once they choose another type's option, there is at least one party making such a commitment.

\subsubsection*{Optimal Contract Design in CML.}
Amalgamating the constraints and the coordinator's objective, we obtain the constrained optimization problem for optimal contract design in CML using models as the rewards:
\begin{align}
      &  \max_{\mathcal{C} \triangleq {(\mathbf{r}_i, m_i)}_{i=1}^{I}}  \; \mathbb{E}_{n\sim \mathrm{Multi}(N,p)}\left[a\left(\sum_{i=1}^I n_i m_i\right)\right] \label{objective} \; \text{s.t.}\\
    & \; \left\{\begin{array}{l}
       \mathbb{E}_{n_i \geq 1}[v(r_i)] - c_i m_i \geq f_i, \forall i\ ; \\
       \mathbb{E}_{n_i \geq 1}[v(r_i)] - c_i m_i \geq \mathbb{E}_{n_j \geq 1}[v(r_j)] - c_i m_j, \forall i,j\ ; \\
       \Vert r(n)\Vert_\infty \leq  a\left(\sum_{i=1}^I n_i m_i\right), \forall n \in \mathrm{Multi}(N,p)\ .
    \end{array}\right. \label{constraints}
\end{align}
Before solving this problem, we detour to analyze a special benchmark where the coordinator has the ability to observe the private information possessed by the participants, which is conventionally called the \textit{complete information} scenario.

\section{Contracting with Observable Costs} \label{Contracting with Observable Costs}
We conduct an analysis of the complete information scenario, as it offers insights on two fronts. Firstly, it helps establish a welfare benchmark, against which we could evaluate the welfare loss incurred by the existence of hidden information. Secondly, it models particular CML settings, in which it is relatively cheap to obtain or elicit the private cost information. When the coordinator can observe a party's type, the IC conditions become redundant, as the coordinator can directly contract parties based on their costs. The problem thus reduces significantly to
\begin{align*}
        \max_{\{r_i, m_i\}_{i=1}^{I}} & \; a\left(\sum_{i=1}^I n_i m_i\right) \\
    \text{s.t.} & \; \left\{\begin{array}{l}
       v(r_i) - c_i m_i - f_i \geq 0\ ,  \forall i\ ; \\
       \\ r_i \leq  a\left(\sum_{i=1}^I n_i m_i\right),\forall i\ .
    \end{array}\right.
\end{align*}
Note that we drop the expectation operators because the coordinator now fully knows the number of parties present for each type. Consequently, the rewards also become deterministic. Solving the problem leads to the following proposition. 

\begin{proposition}
    Under the complete information scenario, the optimal strategy for the principal is to offer the best model to all participating parties and require them utilize the amount of data such that an party’s IR constraint binds. The parties will in general utilize more data than they would when training a model on their own.    
\end{proposition}
The detailed proof is deferred to Appendix A, and we provide here a brief intuition on offering the best model to all. The key is that the full observability of a party's cost eliminates the possibility of cheating. Even if a party wishes to choose the option designed for another type that requires less data contribution, they can longer do so as the coordinator can embed the type into the option and easily verify a party's eligibility at the time of contract signing. Since models are freely replicable, granting the highest rewards to parties incentivizes them to make the highest possible level of contribution while satisfying the IR condition.

While the idea of offering the same model to all parties making different contributes may seem counter-intuitive from a fairness perspective, it actually still obeys the principle that a bigger contributor ends up with a higher profit.
\begin{proposition}
    Under the complete information scenario, a party with lower cost makes more contribution and obtains higher profits, thereby obeying the principle that a bigger contributor ends up with a higher profit.
\end{proposition}
The proof is slightly involved and thus deferred to Appendix A. We conclude this section by noting that the results hold with the goal of maximizing the accuracy of the collectively trained model. In practice, if the coordinator wants to enact additional fairness requirements, such as letting bigger contributors gain more through the CML scheme, she can do so by modifying the IR constraints. We provide such a framework in Appendix B for interested readers.

\section{Contracting with Private Costs}
When data costs are private information of the parties, it could have serious implications on the CML scheme without contract. As we demonstrate in Appendix B, a complete collaboration failure could occur where all parties contribute nothing to the scheme. In other cases, equilibrium does not exist, making the learning outcome unpredictable. Designing a contract helps address these issues but solving the problem defined by (\ref{objective}) and (\ref{constraints}) directly is difficult due to the non-convexity of the constraints and the enormous number of choice variables. Luckily, the problem can be simplified on two fronts to improve its tangibility. Firstly, we transform the original problem into a convex constrained optimization problem with respect to first moments (termed the \textit{first-moment problem}), which significantly reduces the number of choice variables we need to optimize with. We then derive a mapping from the solution of the first-moment problem to one that elegantly solves the original problem. Secondly, we conduct constraint analysis of the first-moment problem, further removing redundant constraints and delineating the properties an optimal contract should satisfy.

\subsection{First-moment problem}
The key to converting the original problem into a first-moment problem lies in the relaxation of the budget constraint. The following provides a necessity result.
\begin{proposition}
    The budget constraint in the original problem implies the budget constraint in first moments. Namely,
    \begin{align*}
        \Vert r(n)\Vert_\infty \leq &  a\left(\sum_{i=1}^I n_i m_i\right) \implies \notag \\  \mathbb{E}_{n_i \geq 1}[v(r_i)] & \leq \mathbb{E}_{n_i \geq 1}\left[v\left(a\left(\sum_{i=1}^I n_i m_i\right)\right)\right], \forall i\ .
    \end{align*}
\end{proposition}
The proof is left to Appendix A. With this relaxation, all constraints in the original optimization problem are related to only the first moments of $v(r_i), \forall i$. For notational brevity, we let $t_i \triangleq \mathbb{E}_{n_i \geq 1}[v(r_i)]$, which denotes the expected revenue gained by a type-$i$ party using the rewarded model from the scheme. The first-moment problem is defined as follows:
\begin{align*}
        \max_{\{(t_i, m_i)_{i=1}^{I}\}} & \; \mathbb{E}_{n\sim \mathrm{Multi}(N,p)}\left[a\left(\sum_{i=1}^I n_i m_i\right)\right] \\
    \text{s.t.} & \; \left\{\begin{array}{l}
       t_i - c_i m_i \geq f_i\ , \forall i\ ; \\
       t_i - c_i m_i \geq t_j - c_i m_j\ , \forall i,j\ ; \\
       t_i \leq  \mathbb{E}_{n_i \geq 1}\left[v\left(a\left(\sum_{i=1}^I n_i m_i\right)\right)\right], \forall i\ .
    \end{array}\right.
\end{align*}
Note that instead of optimizing with respect to the distributions $\mathbf{r}_i, \forall i$, we only need to optimize with respect to scalars $t_i, \forall i$ in the first-moment problem. To complete the transformation, we are left to show that there exists a mapping from the solution of the first-moment problem to that of the original problem. We do this by construction.
\begin{proposition}
    Let $(t^*_i, m^*_i)_{i=1}^I$ denote the solution to the first-moment problem, and $\bar{t}_i \triangleq \mathbb{E}_{n_i \geq 1}\left[v\left(a\left(\sum_{i=1}^I n_i m_i\right)\right)\right]$. Then, the following mapping (\textbf{proportional assignment}) maximizes the original problem defined by (\ref{objective}) and (\ref{constraints}):
    \begin{equation}
        r_i(n) = v^{-1}\left(\frac{t_i^*}{\bar{t}_i}v\left(a\left(\sum_{i=1}^I n_i m_i\right)\right)\right). \label{proportional mapping}
    \end{equation}
\end{proposition}
The proof is deferred to Appendix A, while the idea is intuitive: Regardless of the actual realization of the type distribution, we will assign the realized model rewards according to the ratios defined by the first moments. Note that there exist other assignment policies that solve the original problem while aligning with the first-moment problem solution, while the proportional assignment rule is one that is intuitive and easy to implement in reality.

\subsection{Constraint analysis}
The first-moment problem has in total $(I^2 + I)$ inequality constraints, which gets unwieldy as $I$ rises. As we will see later, many of these constraints turn out to be superfluous, thus inviting a more precise description of the feasible set. 

\subsubsection*{Proportional Fairness.} We start with a useful observation about parties' levels of contribution and rewards in an optimal contract under incomplete information.
\begin{proposition}\label{proportional fairness}
    With incomplete information, a party with lower cost should contribute weakly more data and receive weakly better rewards, i.e., given $c_i > c_{i+1}\ ,$
    \begin{equation*}
        m_{i+1} \geq m_i\quad \text{and}\quad t_{i+1} \geq t_i\ .
    \end{equation*} 
\end{proposition}
The proof makes use of the IC constraints and is postponed to Appendix A. Importantly, Proposition \ref{proportional fairness} establishes a notion of fairness as a necessary condition of an optimal contract, suggesting that a bigger contributor must receive a better reward. This differs from incentive mechanism designs using cooperative game theory (CGT) where this property is treated as a desideratum based on which one devises a reward scheme \citep{sim_collaborative_2020}.

\subsubsection*{Adjacent Comparisons Constraints.}
As each participant needs to compare their tailored option with the rest $(I-1)$ option, we have a copious set of $I(I-1)$ incentive compatibility constraints. The following result shows that the actual key constraints is much smaller in number. The proof is involved and is left to Appendix A for interested readers.
\begin{theorem}
\label{equivalent constraints theorem}
    With incomplete information, the following two sets of constraints are equivalent:
    \begin{enumerate}[label=(\alph*)] 
      \item $t_i - c_i m_i \geq t_j - c_j m_j, \forall i, j\ ;$
      \item $\left\{\begin{array}{l}
      m_{i+1} \geq m_{i}, \forall i \in\{1, \dots, I-1\}\ ;\\
      t_i - c_i m_i = t_{i-1} - c_i m_{i-1}, \forall i \in \{2, \dots, I\}\ .
      \end{array}\right.$
    \end{enumerate}
\end{theorem}

\subsubsection*{Weak Efficiency.}
Proposition \ref{proportional fairness} establishes one connection with incentive mechanism designs in CGT. Here, we show another shared feature between the paradigms.
\begin{proposition}
    With incomplete information, the most cost-efficient type (party with the lowest variable cost) must be offered the best model, i.e., \begin{equation*}
        t_I = \mathbb{E}_{n_I \geq 1}\left[v\left(a\left(\sum_{i=1}^I n_i m_i\right)\right)\right].
    \end{equation*}
\end{proposition}
This result aligns with the concept of weak efficiency proposed by mechanism design studies using CGT \citep{sim_collaborative_2020}. Namely, when employing only model rewards to incentivize collaborative machine learning, the best model must be awarded to one of the participants. Like Proposition \ref{proportional fairness}, weak efficiency in this context is attained as a necessary condition rather than being treated as a desideratum.

\subsubsection*{Highest-Cost Type Break Even Condition.}
Finally, we show that the IR condition must bind for participants with the highest per-unit cost, leaving them no rent to be gained.
\begin{proposition}
\label{break even condtion}
    With incomplete information, a type-$1$ party would obtain utilities no greater than their reservation level if choosing options designed for the other types. In addition, the individual rationality constraint must bind for parties of type $1$. Namely,
    \begin{align*}
        t_j - c_1 m_j \leq & f_1 , \forall j \in \{2, \dots, I\}\ , \\
        t_1 - c_1 m_1 = & f_1\ .
    \end{align*}
\end{proposition}
In other words, the optimal contract that maximizes the expected accuracy of the collectively trained model should make participants with the highest per-unit cost of contribution indifferent between participating and opting out.
\subsubsection*{Simplified First-Moment Problem.}
With the above simplifications, the first-moment problem becomes
\begin{align}
    & \max_{\{(t_i, m_i)_{i=1}^{I}\}} \; \mathbb{E}_{n\sim \mathrm{Mul}(N,p)} \left[ a\left(\sum_{i=1}^I n_i m_i\right)\right] \label{the simplified constrained optimization problem objective} \\
    & \text{s.t.}  \left\{\begin{array}{l}
       t_1 - c_1 m_1 - f_1 = 0\ ; \\
       t_i  \leq \mathbb{E}_{n_i \geq 1} \left[ v\left(a\left(\sum_{i=1}^I n_i m_i\right)\right)\right], \forall i \in \mathcal{I}\ ;\\
       t_i - c_i m_i = t_{i-1} - c_i m_{i-1},  \forall i \in \{2, \dots, I\} ; \\
       m_i \geq m_{i-1}, \forall  i \in \{2, \dots, I\}\ ;\\
       t_i - c_i m_i - f_i \geq 0,   \forall i \in \{2, \dots, I\}\ . 
    \end{array}\right. \label{the simplified constrained optimization problem}
\end{align}
We keep the inequality of the budget constraint for type-$I$ parties so that the resulting problem is still convex (cf. Appendix B). The set of constraints in (\ref{the simplified constrained optimization problem}) fully specifies model rewards as a function of contributions:
\begin{equation*}
    t_i = \left\{\begin{array}{ll}
         f_1 + c_1 m_1 & \text{if}\ i = 1\ ; \\
         f_1 + c_1 m_1 + \sum_{k=2}^i c_k (m_{k} - m_{k-1}) & \text{if}\ i > 1\ . 
    \end{array}\right.
\end{equation*}

The problem defined by (\ref{the simplified constrained optimization problem objective}) and (\ref{the simplified constrained optimization problem}) is convex and can be solved by numerical optimization methods, such as the trust-region interior-point algorithm, which works by staying away from the boundary of the feasible region defined by the inequality constraints and weakening the barrier effects as the estimate of the solution gets increasingly accurate \citep{nocedal_numerical_2006}.

\section{Experiments}
To gain numerical insights into optimal contract design, we conduct a series of experiments with specified forms of the accuracy function and the valuation function. Following \citet{karimireddy_mechanisms_2022}, we adopt the standard generalization bound \citep{Mohri_2018} as the accuracy function, expressed as follows:
\begin{equation*}
    a(m) := \max \left\{0, a_{opt} - \frac{\sqrt{2k(2+\log(m/k))+4}}{\sqrt{m}}\right\},
\end{equation*}
where $m$ measures the quantity of data used for model training; $a_{opt}$ is the optimal accuracy achievable by the model, and $k$ captures the difficulty of the learning task. Arguably, this choice of $a(m)$ itself is not concave due to its piece-wise nature, but this non-concavity can be addressed via the concavity of each piece of the function. We set $k = 1$ and $a_{opt} = 1$ for the experiments. We assume a constant return to the model accuracy, $v(x) := 100 x$, so a model with perfect accuracy is worth $100$ in monetary terms.

\subsection{Two-type Case}
\begin{figure}[t]
    \centering
    \includegraphics[width=\linewidth]{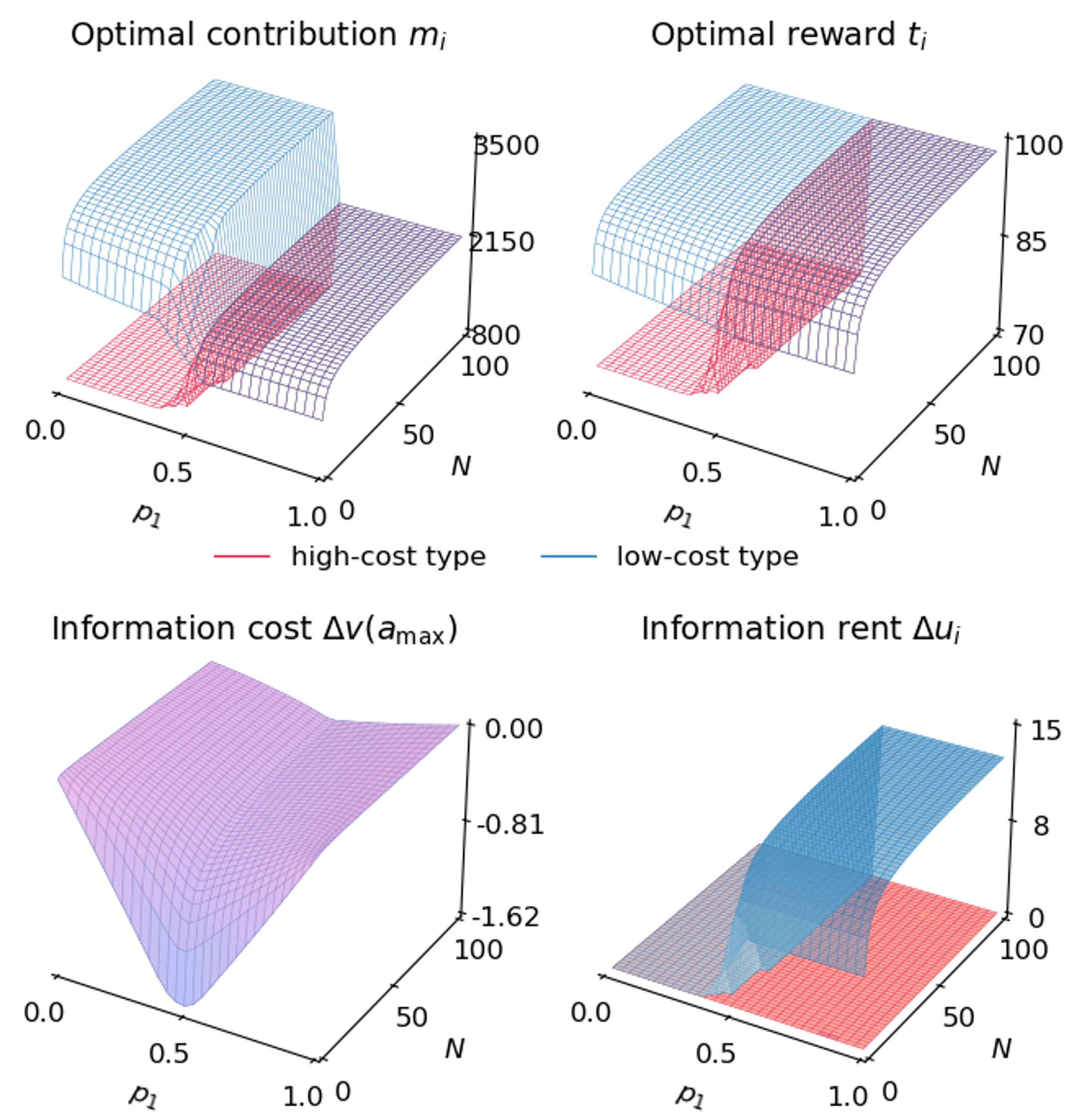}
    \caption{\textbf{Top}: Optimal contracts under incomplete information for varied probability of high-cost type $p_1 \in (0,1)$ and total number of participants $N \in [2, 100]$, with $c=\{0.02, 0.01\}$. \textbf{Bottom}: Information costs for the coordinator and information rents for the parties under incomplete information vis-à-vis complete information.}
    \label{fig:Two-type}
\end{figure}
We first focus on the case where there are two private types of parties, $\mathcal{I} = \{1,2\}$, to analyze how contextual factors such as the total number of participants $N$ and the type distribution $p$ affect the design of optimal contracts. We specify the per-unit data cost for the high-cost type as $c_1 = 0.02$ and that for the low-cost type as $c_2 = 0.01$. In this setting, both types would have initial incentives to train a model on their own without the CML scheme—the high-cost type would use $715.6$ units of data and obtain a model valued at $69.6$ in monetary terms, and the low-cost type would use $1148.5$ units of data to achieve a model reward worth $75.6$. The top panel of Figure \ref{fig:Two-type} depicts the optimal contract for the CML scheme under incomplete information for varied probability of high-cost type $p_1 \in [0,1]$ and total number of participants $N \in [2, 100]$, the solutions of which are obtained by solving the corresponding first-moment problems. Under the incentivized CML scheme, both types contribute more data and obtain better models than their reservation levels. For contract design, ceteris paribus, a higher probability of the high-cost type in the population makes it less favorable to create distinct contract options for the types—making a pooling contract more likely to be optimal from the coordinator's perspective. In contrast, the total number of participants has a relatively marginal effect. When all other factors are held constant, a larger participant pool leads to greater differentiation between the options in a separating contract.

To gauge the welfare implications of the information asymmetry, we calculate the information cost and information rent under incomplete information. The \textit{information cost}, $\Delta v(a_{\max})$, is defined as the expected difference between the value of the collectively trained model under incomplete information and that under complete information:
\begin{equation*}
    \Delta v(a_{\max}) = \mathbb{E}\left[t_C - v\left(a\left(\sum_{i=1}^I  n_i m_i^\mathrm{complete}\right)\right)\right],
\end{equation*}
where $t_C = v(a(\sum_{i=1}^I  n_i m_i))$ is the value of the collectively trained model under incomplete information, and $m_i^\mathrm{complete}$ is the required data contribution from a type-$i$ party under complete information, which varies with different realizations of $n$. Similarly, the \textit{information rent}, $\Delta u_i$, is defined as the expected utility surplus of the party under incomplete information vis-à-vis complete information:
\begin{equation*}
    \Delta u_i = \mathbb{E}\left[u_i(t_i, m_i) - u_i\left(v(r_i^\mathrm{complete}), m_i^\mathrm{complete}\right)\right],
\end{equation*}
where $(t_i, m_i)$ is the contract option designed for a type-$i$ party under incomplete information and $v(r_i^\mathrm{complete})$ is the model reward given to type-$i$ under complete information, which depends on the realization of $n$. The bottom panel of Figure \ref{fig:Two-type} shows the information costs and information rents for the two-type setting for varied $p_1$ and $N$. Aligned with general intuition, information asymmetry affects the collaboratively trained model most conspicuously when the number of participants is low and the probabilities of different types are similar. When one private type becomes dominant in the population, the optimal contract design accommodates by asking the corresponding type to contribute more at the sacrifice of the contribution from the other type, narrowing the gap with the solution under complete information. The low-cost type earns information rent under pooling contracts due to its ability to satisfy the stipulated contribution requirements at a lower cost than its high-cost counterparts.

\subsection{Multi-type Case}
\begin{figure}[!htb]
    \centering
    \includegraphics[width=\linewidth]{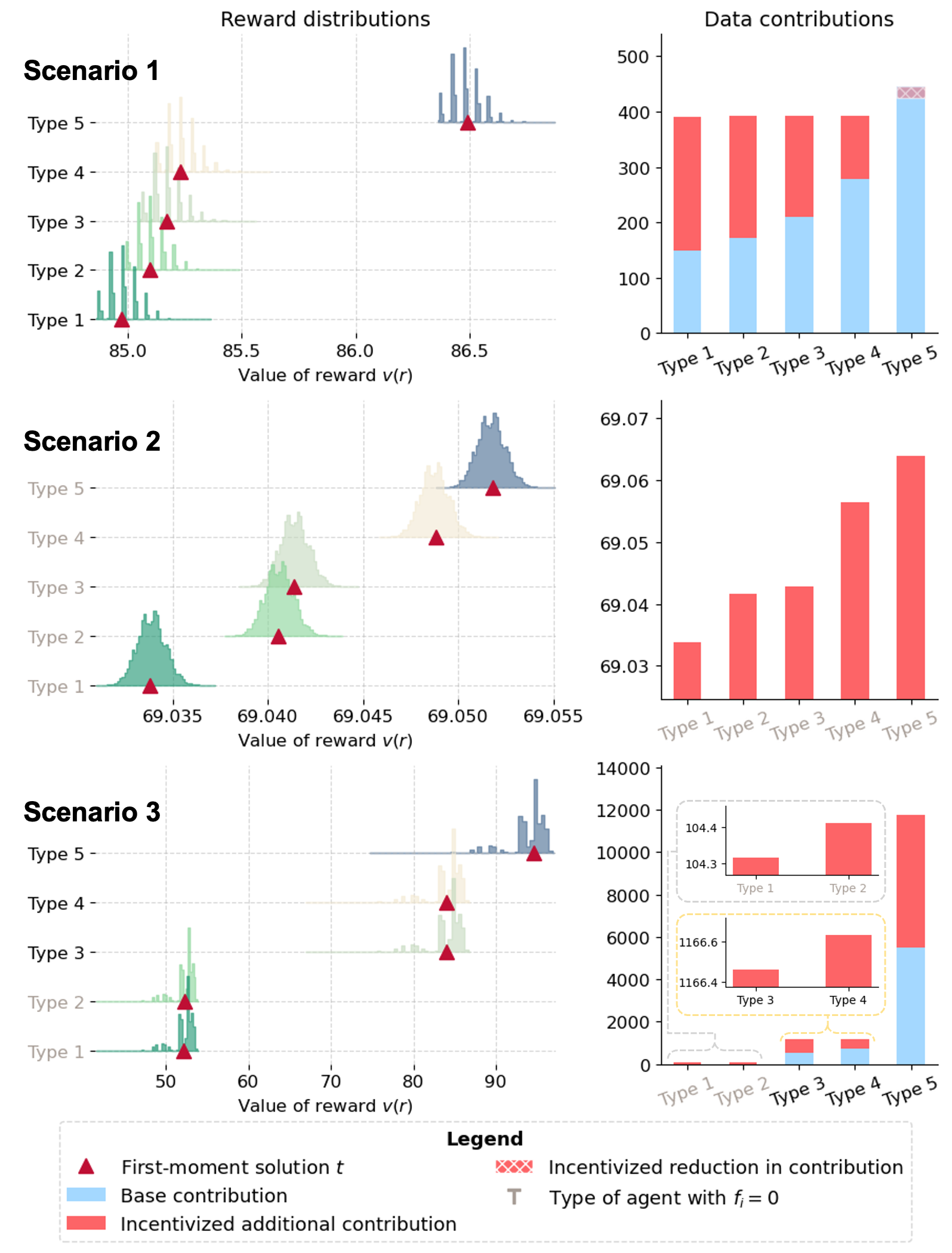}
    \caption{Optimal contract designs for multi-type scenarios. \textbf{Scenario 1}: All types would train a model on their own. \textbf{Scenario 2}: All types would not train a model on their own due to prohibitive costs. \textbf{Scenario 3}: Some types  would train the model on their own and others would not.}
    \label{fig:Multitype case}
\end{figure}
Next, we consider more general cases where there are more than two private types in the population. For the ease of comparisons, we set $N=10$ and $I=5$, $p_i = 0.2, \forall i \in \mathcal{I}$ but vary the private costs $c$. We consider three different scenarios: 1) all types find it in their interest to train a model on their own, with $c=\{0.2, 0.16, 0.12, 0.08, 0.04\}$; 2) all types would not train a model on their own due to high per-unit costs, with $c = \{1, 0.85, 0.7, 0.55, 0.4\}$; 3) some types would train the model on their own and others would not, with $c = \{0.5, 0.4, 0.03, 0.02, 0.001\}$. Figure \ref{fig:Multitype case} shows the simulation results for the three scenarios. We highlight some of the key observations below.

\paragraph{A party may be incentivized to contribute less than their reservation level.} This happens with the lowest-cost type (Type 5) in the first simulated scenario. This intriguing result is partially driven by the incentive compatibility constraint—as a higher contribution requirement would cause the type to deviate to the contract option designed for the adjacent type (cf. Appendix B for a graphical illustration). 

\paragraph{Incentivized collaborative scheme can democratize machine learning.} Row 2 of Figure \ref{fig:Multitype case} considers the case when the difficulty of the learning task prohibits all types from training a model on their own. With an effectively designed contract, this hurtle is overcome, with the cost of data collection shared among the participants and each of them receiving a model with decent accuracy as the reward.

\paragraph{In the presence of dominant players, small players can still gain from collaboration.} The last scenario illustrated in Figure \ref{fig:Multitype case} conveys the idea that incentivized collaboration is still possible despite significant differences in cost structures among the participants. In the simulated scenario, the per-unit data contribution cost of a type-1 party is 500 times the cost of a type-5 party. Yet, by contributing $104.3$ units of data, a type-1 party can receive a model reward valued at $52.2$, showing the trickle-down effect of the collaboration.

\section{Conclusion}
In this work, we consider optimal contract design for CML with models as the rewards. We convert the original non-convex problem of optimizing with reward distributions into one solvable through convex constrained optimization algorithms. Our constraint analysis establishes the necessary conditions for an optimal contract. We further demonstrate the framework through numerical experiments, showing its ability to overcome high cost of model training and improve participant welfare. Our findings highlight that optimal contract design is a viable tool for democratizing future technology in an incentive-driven economy. Future research could explore relaxing the distribution assumption to improve scalability. For a detailed discussion on future research directions, we refer interested readers to Appendix B. 

\section*{Acknowledgments}
This research is supported by the National Research Foundation Singapore and DSO National Laboratories under the AI Singapore Programme (AISG Award No: AISG2-RP-2020-018). The authors extend their gratitude to Dr. Wenjie Feng, Dr. Wenbo Zhao, Mingzhe Du, and four anonymous reviewers for their invaluable feedback on this work.

\bibliography{aaai25}
\appendix

\onecolumn
\section{Theoretical Proofs}
\subsection{Related Work}
\setcounter{proposition}{-1}
In the related work section, we posit that maximizing the accuracy of the collectively trained model and maximizing the total amount of collected data in general lead to different optimal contracts under incomplete information. We show it through the following proposition.
\begin{proposition}
\label{non-equivalence}
    Given that $x$ is a vector of choice variables and $g(x;\theta)$ is a linear function of $x$, where $\theta$ is a vector of random variables. Suppose that $f(\cdot)$ is an increasing function. Then,
    \begin{equation}
        \arg\max_x \mathbb{E}_\theta[g(x;\theta)] \not
        \Rightarrow \arg\max_x \mathbb{E}_\theta[f(g(x;\theta))].
    \end{equation}
\end{proposition}
\begin{proof}
    We prove the above result by providing a counterexample. Consider the binomial distribution, with $p_1 = 0.6, p_2 = 0.4$ and $N = 2$. Here $\theta = (n_1, n_2)$. The choice variable is $m_1 \in [0,5]$. Define $g$ and $f$ as follows:
    \begin{align}
        g(m_1; n_1, n_2) & = n_1 m_1 + n_2 (5 -m_1) \\
        f(z) & = \sqrt{z}
    \end{align}
    \begin{table}[!htp]
        \centering
        \begin{tabular}{c c| c || c | c}
        \hline
        $n_1$  &  $n_2$ & p & $g(m_1; n_1, n_2)$ & $ f(g(m_1; n_1, n_2)) $\\
        \hline
        2 & 0 & 0.36 & $2 m_1$ & $\sqrt{2 m_1}$\\
        1 & 1 & 0.48 & $5$ & $\sqrt{5}$\\
        0 & 2 & 0.16 & $10 - 2 m_1$  & $\sqrt{10 - 2 m_1}$ \\
        \hline
        \end{tabular}
        \caption{Counterexample showing Proposition 9.}
        \label{tab:p9-counterexample}
    \end{table}
    
    \noindent First, we find $m_1$ that gives the maximum of $\mathbb{E}_\theta[ g(m_1; n_1, n_2)]$:
    \begin{equation}
        \mathbb{E}_\theta[ g(m_1; n_1, n_2)] = 0.4 m_1 + 4
    \end{equation}
    Thus, the maximum is given by $\bar{m}_1 = 5$.
    Then, consider $\mathbb{E}_\theta[f(g(m_1; n_1, n_2))]$:
    \begin{equation}
        \mathbb{E}_\theta[f(g(m_1; n_1, n_2))] = 0.36 \sqrt{2m_1} + 0.48 \sqrt{5} + 0.16 \sqrt{10-2m_1}
    \end{equation}
    Taking the derivative of function and setting it to $0$, we derive
    $m_1^* = \frac{405}{97}$.
    \begin{equation}
        \bar{m}_1 > m_1^*.
    \end{equation}
    \begin{figure}[!htb]
        \centering
        \includegraphics[width = 0.5 \textwidth]{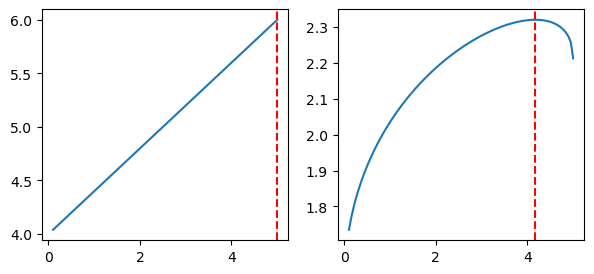}
        \caption{Functions $f$ and $g$ and the $m_1$'s that give the maximum values.}
        \label{fig:accu vs data volume}
    \end{figure}
\end{proof}
The implication of Proposition \ref{non-equivalence} is more significant than it may seem. At face value, it implies that maximising the expected data volume is not the same as maximising the expected model accuracy, unlike the case when expectation is not used (i.e., the complete information scenario). In the case of two types of parties, maximising the expected data volume encourages corner solutions (depending on the probabilities of the two types), whereas maximising the expected curvature may not. The greater the curvature, the greater the subduing effect.

It is worth noting that maximizing the total amount of collected data is a special case of our framework, where $a(\cdot)$ is chosen to be the identity mapping, i.e. $a(x) = x$, in the objective function.
\subsection{Problem Setup}

\begin{proposition}[Reservation utilities]
\label{reservation utility}
Let $\bar{m}_i$ denote the data contribution a type-$i$ party is willing to commit to when training a model on their own. If $c_i \leq c_j$, then $\bar{m}_i \geq \bar{m}_j$, and $f_i \geq f_j$.
\end{proposition}

\begin{proof}
The optimization problem when a party of type $i$ trains a model on their own is:
\begin{equation}
    \max_{m_k} \tilde{u}_k = v(a(m_k)) - c_k m_k
\end{equation}
For simplicity, we drop the $m_k\geq 0$ constraint, as we can threshold the solution at $0$ when $m_k <0$. Denote $\tilde{v}(m_k)\triangleq v(a(m_k))$. For a non-trivial solution $\bar{m}_k > 0$, the optimal solution is given by the first-order condition (FOC):
\begin{equation}
    \tilde{v}^\prime(\bar{m}_k)= c_k. \label{reserv contribution}
\end{equation}
As $a^{\prime\prime}(\cdot)<0$ and $v^{\prime\prime}(\cdot)\leq0$, we have $\tilde{v}^{\prime\prime}(\cdot) <0$. In other words, $\tilde{v}^\prime(\cdot)$ is a decreasing function of $m_k$. Since $c_i \leq c_j$, it follows that $\bar{m}_i \geq \bar{m}_j$. For $c_h> \tilde{v}^\prime(0) \geq \tilde{v}^\prime(\bar{m}_k) = c_k, \forall \bar{m}_k >0$, we have $\bar{m}_h = 0 \leq  \bar{m}_k$. This completes the first half of the proof.

\medskip

Now we show that $f_i \geq f_j$ by utilizing implicit differentiation. First, we write the FOC as
\begin{equation}
    \tilde{v}^\prime(\bar{m}_k(c_k)) = c_k,
\end{equation}
noting that it defines an implicit function that maps $c_k$ to $\bar{m}_k$.

Differentiate both sides with respect to $c_k$ and we get
\begin{align}
    \tilde{v}^{\prime\prime}(\bar{m}_k) \bar{m}^{\prime}_k(c_k) = 1 \\
    \bar{m}^{\prime}_k(c_k) = \frac{1}{\tilde{v}^{\prime\prime}(\bar{m}_k)}.
\end{align}

The reservation utility is defined as the utility level achieved with $\bar{m}_k$, i.e.,
\begin{equation}
    f_k = \tilde{u}_k(\bar{m}_k) = \tilde{v}(\bar{m}_k) - c_k \bar{m}_k
\end{equation}
Now, differentiate the reservation utility with respect to the idiosyncratic cost parameter and we get
\begin{align}
    \frac{\partial f_k}{\partial c_k} = \frac{\tilde{v}^{\prime}(\bar{m}_k)}{\tilde{v}^{\prime\prime}(\bar{m}_k)} - \bar{m}_k - \frac{c_k}{\tilde{v}^{\prime\prime}(\bar{m}_k)} = - \bar{m}_k \leq 0.
\end{align}
The second equality is obtained by applying the FOC. Therefore, $c_i \leq c_j$ implies that $f_i \geq f_j$. Note that for $c_h > \tilde{v}^\prime(0)$, $f_h = 0$. The result continues to hold.
\end{proof}

\subsection{Screening with Observable Costs}

\begin{proposition}[Optimal contract under complete information]
    Under the complete information scenario, the optimal strategy for the principal is to offer the best model to all participating parties and require utilizing the amount of data such that a party’s IR constraint binds. The parties will in general utilize more data than they would when training a model on their own.
\end{proposition}
\begin{proof}
    The constrained optimisation problem is 
    \begin{align}
        & \max_{{\{r_i, m_i\}}_{i=1}^I} a\left(\sum_{i=1}^I n_i m_i\right) \\
        \text{s.t.   } & \left\{\begin{array}{l}
        v(r_i)- c_i m_i - f_i \geq 0,  \forall i;  \\
        r_i \leq a\left(\sum_{i=1}^I n_i m_i\right),  \forall i.   
        \end{array}\right.
    \end{align}
    The Lagrangian is
    \begin{equation}
        \mathcal{L} = a\left(\sum_{i=1}^I n_i m_i\right) + \sum_i \lambda_i ( v(r_i)- c_i m_i - f_i) - \sum_i \mu_i \left[r_i - a\left(\sum_{i=1}^I n_i m_i\right)\right]
    \end{equation}
    \begin{align}
        \textbf{Stationarity}  & \; \; \; \frac{\partial \mathcal{L}}{\partial m_i} = a^\prime\left(\sum_{i=1}^I n_i m_i\right)n_i - \lambda_i c_i + \mu_i a^\prime \left(\sum_{i=1}^I n_i m_i\right)n_i
        =0; \label{S1}\\
        & \; \; \; \frac{\partial \mathcal{L}}{\partial r_i} = \lambda_i v^\prime(r_i) - \mu_i = 0; \label{S2}\\
        \textbf{Primal feasibility} & \; \; \; v(r_i) - c_i m_i - f_i \geq 0, \; \forall i; \label{PF1}\\
        & \; \; \; r_i \leq a\left(\sum_{i=1}^I n_i m_i\right), \; \forall i; \label{PF2}\\
        \textbf{Dual feasibility} & \; \; \; \lambda_i \geq 0, \; \forall i; \label{DF1}\\
        & \; \; \; \mu_i \geq 0, \; \forall i; \label{DF2}\\
        \textbf{Complementary slackness} & \;\;\; \lambda_i ( v(r_i) - c_i m_i - f_i) = 0, \; \forall i; \label{CS1}\\
         & \;\;\;   \mu_i \left(r_i - a\left(\sum_{i=1}^I n_i m_i\right)\right) = 0, \; \forall i. \label{CS2}
    \end{align}
    \underline{Claim 1}: \textit{(\ref{PF2}) must bind.} \smallskip \\ Suppose not. Then (\ref{DF2}) and (\ref{CS2}) imply $\mu_i = 0$. As $v^\prime(\cdot) > 0$ by assumption, (\ref{S2}) and (\ref{DF1}) imply $\lambda_i = 0$. This creates a contradiction between (\ref{S1}) and the assumption that $a^\prime(\cdot)>0$. Thus, (\ref{PF2}) must bind, i.e.,
    \begin{equation}
        r_i = a\left(\sum_{i=1}^I n_i m_i\right), \; \forall i. \label{PF1-bind}
    \end{equation}
    \underline{Claim 2}: \textit{(\ref{PF1}) must bind.} \smallskip \\ Suppose not. We get $\lambda_i = 0$, which combined with (\ref{S2}) gives $\mu_i = 0$. Again, a contradiction arises between (\ref{S1}) and the assumption that $a^\prime(\cdot)>0$. Therefore, (\ref{PF1}) must bind, i.e., 
    \begin{equation}
        v(r_i) - c_i m_i - f_i = 0, \; \forall i. \label{PF2-bind}
    \end{equation}
    Solving (\ref{PF1-bind}) and (\ref{PF2-bind}), we can derive the optimal strategy ${\{r_i^*, m_i^*\}}_{i=1}^I$ for the principal under complete information. Note that it can be shown that the second-order conditions are satisfied, the details of which are left to the reader. \medskip \\ 
    \underline{Claim 3}: \textit{An optimal strategy exists for the principal under complete information. In other words, there is a solution to the system formed by (\ref{PF1-bind}) and (\ref{PF2-bind})}. \smallskip \\
    Finding a solution to the system is equivalent to finding a solution to the following equations
    \begin{equation}
        v\left(a\left(\sum_{i=1}^I n_i m_i\right)\right) - c_i m_i = f_i, \forall i.
    \end{equation}
    On the one hand, note that $a(\cdot)$ is bounded above by 1, but $c_i m_i$ is unbounded. Thus, there exists $M_i$ large enough such that
    \begin{equation}
        v\left(a\left(\sum_{j\neq i}^I n_j m_j + n_i M_i\right)\right) - c_i M_i < 0.
    \end{equation}
    On the other hand, by definition (from Proposition 1) we have
    \begin{equation}
        v\left(a\left(\bar{m}_i\right)\right) - c_i \bar{m}_i = f_i.
    \end{equation}
    As both $v(\cdot)$ and $a(\cdot)$ are increasing functions, it follows that
    \begin{equation}
        v\left(a\left(\sum_{j\neq i}^I n_j m_j + n_i \bar{m}_i\right)\right) - c_i \bar{m}_i > f_i.
    \end{equation}
    By the intermediate value theorem, $\exists \, m_i^* \in (\bar{m}_i, M_i) $ such that
    \begin{equation}
        v\left(a\left(\sum_{j\neq i}^I n_j m_j + n_i m_i^*\right)\right) - c_i m_i^* = f_i. \label{auxiliary1}
    \end{equation}
    Note that in the above proof, the result holds for any $m_j$ where $j\neq i$. In particular, setting $m_j =m_j^*, \; \forall j$ in (\ref{auxiliary1}), it follows that $\exists \, \{m_i^*\}_{i=1}^I$ with $m_i^* > \bar{m}_i$ such that
    \begin{equation}
        v\left(a\left(\sum_{i=1}^I n_i m_i^*\right)\right) - c_i m_i^* = f_i. \label{optimal solution}
    \end{equation}
\end{proof}

\begin{proposition}[Welfare implications under complete information]
    Under the complete information scenario, a party with lower cost makes more contribution and obtains higher profits, thereby obeying the principle that a bigger contributor ends up with a higher profit.
\end{proposition}
\begin{proof}
    The second part of the proposition is readily implied from Propositions 1 and 2. From Proposition 2, we know that IR constraints are binding under complete information. Therefore, for type-$i$ party, they end up getting the reservation utility $f_i$. From Proposition 1, we know that $f_i \geq f_j$ for $c_i \leq c_j$. We concentrate on proving the first part of the proposition, i.e., a party with lower cost makes more contribution in the complete information scenario.

    From Proposition 2, we know that the optimal $m_i^*$ for type-$i$ party satisfies
    \begin{align}
        v\left(a\left(\sum_{i=1}^I n_i m_i^*\right)\right) - c_i m_i^* = & \; f_i, \\
        v\left(a\left(\sum_{i=1}^I n_i m_i^*\right)\right) - c_i m_i^* = & \; v(a(\bar{m}_i)) - c_i \bar{m}_i.
    \end{align}
    Treat $m_i^*$ and $\bar{m}_i$ as implicit functions of $c_i$, differentiate both sides of the equation by $c_i$, and we get
    \begin{align}
        v^\prime\left(a\left(\sum_{i=1}^I n_i m_i^*\right)\right) a^\prime \left(\sum_{i=1}^I n_i m_i^*\right) n_i \frac{d m_i^*}{d c_i} - m_i^* - c_i \frac{d m_i^*}{d c_i}  = v^\prime(a(\bar{m}_i)) a^\prime(\bar{m}_i)\frac{d \bar{m}_i}{d c_i} - \bar{m}_i - c_i \frac{d \bar{m}_i}{d c_i}.
    \end{align}
    Collecting the terms and using auxiliary results from Proposition 1, we get
    \begin{align}
        \left[v^\prime\left(a\left(\sum_{i=1}^I n_i m_i^*\right)\right) a^\prime \left(\sum_{i=1}^I n_i m_i^*\right) n_i - c_i\right]\frac{d m_i^*}{d c_i}  = m_i^* - \bar{m}_i \geq 0.
    \end{align}    
    The last inequality follows from Proposition 2. Therefore, what is left is to determine the sign of
    \begin{equation}
        v^\prime\left(a\left(\sum_{i=1}^I n_i m_i^*\right)\right) a^\prime \left(\sum_{i=1}^I n_i m_i^*\right) n_i - c_i.
    \end{equation}

    \noindent \underline{Claim 1}: \textit{$v^\prime\left(a\left(\sum_{i=1}^I n_i m_i^*\right)\right) a^\prime \left(\sum_{i=1}^I n_i m_i^*\right) n_i - c_i \leq 0.$ } \\
    We prove the claim by contradiction. Assume that $v^\prime\left(a\left(\sum_{i=1}^I n_i m_i^*\right)\right) a^\prime \left(\sum_{i=1}^I n_i m_i^*\right) n_i - c_i > 0.$ Recall that the type-$i$ participant's utility function is
    \begin{equation}
        u_i(r_i, m_i) = v(r_i) - c_i m_i.
    \end{equation}
    Now that under complete information $r_i = a\left(\sum_{i=1}^I n_i m_i\right)$, we could write the utility function solely in terms of $m_i$:
    \begin{equation}
        \breve{u}_i(m_i) = v\left(a\left(\sum_{i=1}^I n_i m_i\right)\right) - c_i m_i.
    \end{equation}
    The assumption is equivalent to
    \begin{equation}
        \frac{d \breve{u}_i}{d m_i}(m_i^*) > 0.
    \end{equation}
    This implies that $\exists \; \hat{m}_i > m_i^*$ such that $\breve{u}(\hat{m}_i) > \breve{u}(m_i^*) = f_i.$
    As $a(\cdot) \leq 1$, we have
    \begin{equation}
        \breve{u}_i(m_i) \leq v(1) - c_i m_i,
    \end{equation}
    the right side of which goes to $-\infty$ as $m_i$ approaches $+ \infty$. In particular, let $M_i \triangleq \frac{v(1)}{c_i}$, then
    \begin{equation}
        \breve{u}_i(M_i) \leq 0.
    \end{equation}
    By the intermediate value theorem, $\exists \; \breve{m}_i \in (\hat{m_i}, M_i)$ such that $ \breve{u}_i(\breve{m}_i) = f_i$. But as $ \breve{m}_i > m_i^*$ and $a^\prime(\cdot) > 0$,
    \begin{equation}
        a\left(\sum_{j \neq i} n_j m_j^* + n_i \breve{m}_i\right) >  a\left(\sum_{k=1}^I n_k m_k^*\right).
    \end{equation}
    Therefore, ${\{(r_i^*, m_i^*)\}}_{i=1}^I$ cannot be a solution to the optimization problem—a contribution.
    
    \medskip
    
    \noindent \underline{Claim 2}: \textit{$v^\prime\left(a\left(\sum_{i=1}^I n_i m_i^*\right)\right) a^\prime \left(\sum_{i=1}^I n_i m_i^*\right) n_i - c_i < 0.$ } \\
    Again, we prove by contradiction. Assume that $v^\prime\left(a\left(\sum_{i=1}^I n_i m_i^*\right)\right) a^\prime \left(\sum_{i=1}^I n_i m_i^*\right) n_i - c_i = 0.$ This is equivalent to assuming that
    \begin{equation}
        \frac{d \breve{u}_i}{d m_i}(m_i^*) = 0.
    \end{equation}
    Since $\frac{d^2 \breve{u}_i}{{(d m_i)}^2} < 0$, $m_i = m_i^*$ achieves the global maximum of $\breve{u}_i(m_i)$. As $\breve{u}_i(m_i^*) = f_i$ and $m_i^* > \bar{m}_i$ (from the previous proposition), we have
    \begin{equation}
        f_i = \breve{u}_i(m_i^*) > \breve{u}_i(\bar{m}_i) = v\left(a\left(\sum_{i=1}^I n_i \bar{m}_i\right)\right) - c_i \bar{m}_i \geq v(a(\bar{m}_i)) - c_i \bar{m}_i = f_i.
    \end{equation}
    Therefore, it creates a contradiction.

    \medskip
    
    In conclusion, $v^\prime\left(a\left(\sum_{i=1}^I n_i m_i^*\right)\right) a^\prime \left(\sum_{i=1}^I n_i m_i^*\right) n_i - c_i < 0.$ This indicates that
    \begin{equation}
        \frac{d m_i^*}{d c_i} \leq 0.
    \end{equation}
    Thus, if $c_i \leq c_j$, $m_i^* \geq m_j^*$.
\end{proof}

\subsection{Screening with Private Costs}
\begin{proposition}[Relaxation of the budget constraint]
    The budget constraint in the original problem implies the budget constraint in first moments. Namely,
    \begin{equation}
        \Vert r(n)\Vert_\infty \leq  a\left(\sum_{i=1}^I n_i m_i\right), \forall n \in \mathrm{Multi}(N,p)
        \implies\mathbb{E}_{n_i \geq 1}[v(r_i)] \leq \mathbb{E}_{n_i \geq 1}\left[v\left(a\left(\sum_{i=1}^I n_i m_i\right)\right)\right], \forall i.
    \end{equation}
\end{proposition}
\begin{proof}
    By contraposition. To prove the above proposition, we only need to show that 
    \begin{equation}
        \mathbb{E}_{n_k \geq 1}[v(r_k)] >\mathbb{E}_{n_k \geq 1}\left[v\left(a\left(\sum_{i=1}^I n_i m_i\right)\right)\right], \exists k \in \mathcal{I}.
        \implies \Vert r(\bar{n})\Vert_\infty >  a\left(\sum_{i=1}^I \bar{n}_i m_i\right), \exists \bar{n} \in \mathrm{Multi}(N,p).
    \end{equation}
    Assume not, i.e., $\Vert r(n)\Vert_\infty \leq  a\left(\sum_{i=1}^I n_i m_i\right), \forall n \in \mathrm{Multi}(N,p)$. It follows that
    \begin{equation}
        v(r_k(n)) \leq \Vert v(r(n))\Vert_\infty \leq  v\left(a\left(\sum_{i=1}^I n_i m_i\right)\right), \forall n \in \mathrm{Multi}(N,p).
    \end{equation}
    Consequently,
    \begin{equation}
        \mathbb{E}_{n_i \geq 1}[v(r_i)] \leq \mathbb{E}_{n_i \geq 1}\left[v\left(a\left(\sum_{i=1}^I n_i m_i\right)\right)\right], \forall i.
    \end{equation}
    This contradicts the condition that $\mathbb{E}_{n_k \geq 1}v(r_k)] >\mathbb{E}_{n_k \geq 1}\left[v\left(a\left(\sum_{i=1}^I n_i m_i\right)\right)\right], \exists k \in \mathcal{I}$.
\end{proof}
Using the above result, we can convert the original problem into a much simpler problem in first moments. For notational simplicity, let $t_i \triangleq \mathbb{E}_{n_i \geq 1}[v(r_i)]$. Then, the simplified problem is:
\begin{align}
        \max_{\{(t_i, m_i)_{i=1}^{I}\}} & \; \mathbb{E}_{n\sim \mathrm{Multi}(N,p)}\left[a\left(\sum_{i=1}^I n_i m_i\right)\right] \\
    \text{s.t.} & \; \left\{\begin{array}{ll}
       t_i - c_i m_i \geq f_i  &  \forall i; \\
       t_i - c_i m_i \geq t_j - c_i m_j & \forall i,j; \\
       t_i \leq  \mathbb{E}_{n_i \geq 1}\left[v\left(a\left(\sum_{i=1}^I n_i m_i\right)\right)\right], & \forall i.
    \end{array}\right.
\end{align}
Note that the feasible set of this problem is much larger than that of the original problem. If we find an optimal solution to this problem and then show that it additionally satisfies the original budget constraint. Then we are done.

As the solution to this simplified problem only puts constraints on the first moments, the key to ensure the validity of the original budget constraint is to find a series of mappings from the realised accuracy to the reward (Figure \ref{fig:Venn diagram of problem conversion}). Denote the solution to the simplified problem as $(t^*_i, m^*_i)_{i=1}^n$, and $\bar{t}_i \triangleq \mathbb{E}_{n_i \geq 1}\left[v\left(a\left(\sum_{i=1}^I n_i m_i\right)\right)\right]$. Then, one mapping (we call \textbf{proportional assignment}) that meets the requirement is
\begin{equation}
        r_i(n) = v^{-1}\left(\frac{t_i^*}{\bar{t}_i}v\left(a\left(\sum_{i=1}^I n_i m_i\right)\right)\right). \label{appendix: proportional assignment}
    \end{equation}
\begin{figure}[!htb]
    \centering
    \includegraphics[width=0.5\linewidth]{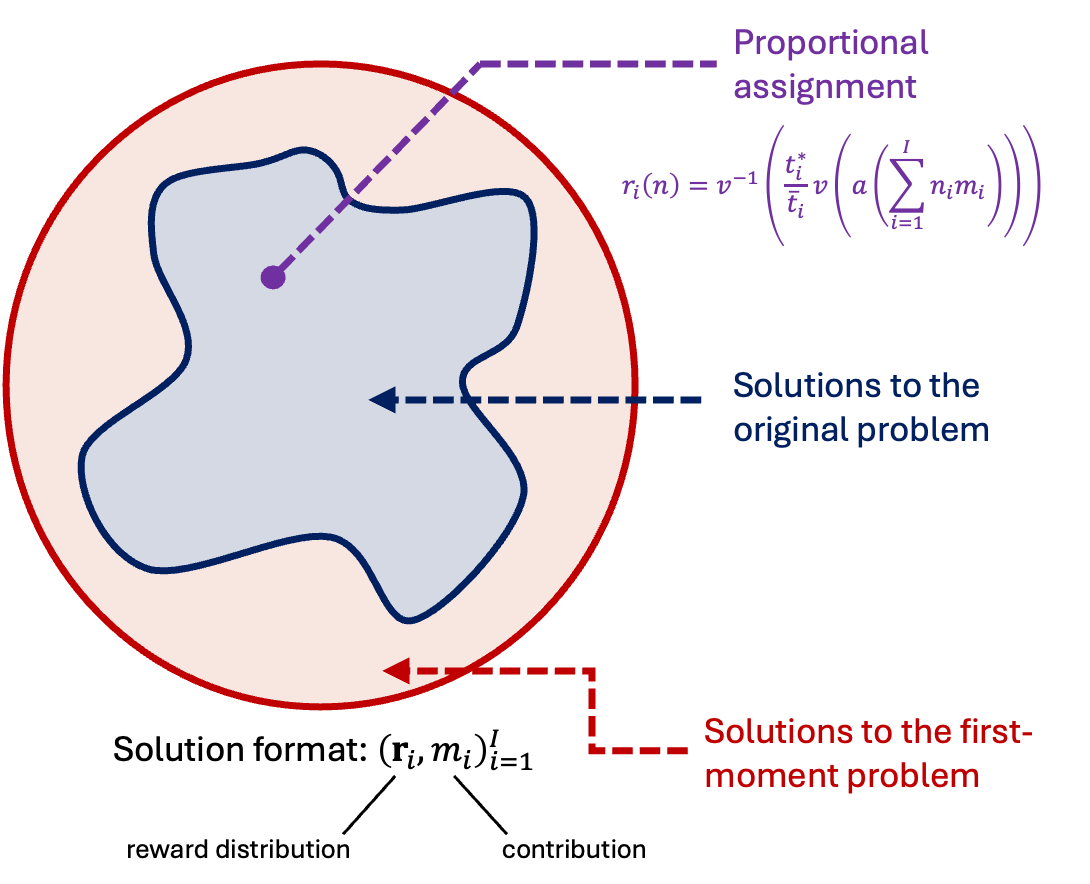}
    \caption{The solution strategy: delineate solutions to the first-moment problem and then use the proportional assignment to solve the original problem.}
    \label{fig:Venn diagram of problem conversion}
\end{figure}

The proof is our focus next. But first, we establish the original budget constraint in its equivalent pecuniary terms.

\setcounter{lemma}{0}

\begin{lemma}[Equivalent transformation of the budget constraint]
\label{Equivalent transformation of the budget constraint}
    The budget constraint in the revamped problem is equivalent to the budget constraint in pecuniary terms. Namely,
    \begin{equation}
        \Vert r(n)\Vert_\infty \leq  a\left(\sum_{i=1}^I n_i m_i\right) \iff 
        \Vert v(r(n))\Vert_\infty \leq  v\left(a\left(\sum_{i=1}^I n_i m_i\right)\right).
    \end{equation}
\end{lemma}
\begin{proof}
    $\implies$ Suppose not, i.e., $\Vert v(r(n))\Vert_\infty > v\left(a\left(\sum_{i=1}^I n_i m_i\right)\right)$. Let $v(r_k(n)) \triangleq \Vert v(r(n))\Vert_\infty$.
    \begin{equation}
        v(r_k(n)) > v\left(a\left(\sum_{i=1}^I n_i m_i\right)\right)
    \end{equation}
    and $v^\prime(\cdot)$ imply that
        \begin{equation}
        r_k(n) > a\left(\sum_{i=1}^I n_i m_i\right),
    \end{equation}
    which contradicts the assumption that $\Vert r(n)\Vert_\infty \leq  a\left(\sum_{i=1}^I n_i m_i\right)$.

    $\impliedby$ Suppose not, i.e., $\Vert r(n)\Vert_\infty >  a\left(\sum_{i=1}^I n_i m_i\right)$. Let $r_k(n) \triangleq \Vert r(n)\Vert_\infty$.

    \begin{equation}
        r_k(n) > a\left(\sum_{i=1}^I n_i m_i\right)
    \end{equation}
    and $v^\prime(\cdot)$ imply that
    \begin{equation}
        v(r_k(n)) > v\left(a\left(\sum_{i=1}^I n_i m_i\right)\right),
    \end{equation}
    which contradicts the assumption that $\Vert v(r(n))\Vert_\infty \leq  v\left(a\left(\sum_{i=1}^I n_i m_i\right)\right)$.
\end{proof}

\begin{proposition}[Proportional assignment solves the original problem]
    Let $(t^*_i, m^*_i)_{i=1}^n$ denote the solution to the first-moment problem, and $\bar{t}_i \triangleq \mathbb{E}_{n_i \geq 1}\left[v\left(a\left(\sum_{i=1}^I n_i m_i\right)\right)\right]$. Then, the following mapping (\textbf{proportional assignment}) maximizes the original problem defined by (\ref{objective}) and (\ref{constraints}):
    \begin{equation}
        r_i(n) = v^{-1}\left(\frac{t_i^*}{\bar{t}_i}v\left(a\left(\sum_{i=1}^I n_i m_i\right)\right)\right).
    \end{equation}
\end{proposition}
\begin{proof}
    By Lemma \ref{Equivalent transformation of the budget constraint}, we only need to prove that proportional assignment satisfies the equivalent budget constraint. Namely,
    \begin{equation}
        \Vert v(r(n))\Vert_\infty \leq v\left(a\left(\sum_{i=1}^I n_i m_i\right)\right),
    \end{equation}
    where $r(n)\triangleq [r_1(n),\dots, r_I(n)]$.
    
    It suffices to show that
    \begin{equation}
         v(r_i(n))\leq v\left(a\left(\sum_{i=1}^I n_i m_i\right)\right), \forall i.
    \end{equation}
    Since $(t^*_i, m^*_i)_{i=1}^n$ is the solution to the simplified problem, it holds that
    $t^*_i \leq \bar{t}_i, \forall i$ (budget constraint in first moments). Then,
    \begin{equation}
        v(r_i(n)) = \frac{t_i^*}{\bar{t}_i}v\left(a\left(\sum_{i=1}^I n_i m_i\right)\right) \leq v\left(a\left(\sum_{i=1}^I n_i m_i\right)\right),
    \end{equation}
    as desired.
\end{proof}

\begin{proposition}[Proportional fairness]
    With incomplete information, a party with lower cost should utilize weakly more data and receive weakly better rewards, i.e., given $c_i > c_{i+1}$,
    \begin{equation}
        m_{i+1} \geq m_i \text{ and } t_{i+1} \geq t_i .
    \end{equation}
\end{proposition}
\begin{proof}
    From the IC constraints, we have
    \begin{align}
        t_i - c_i m_i & \geq t_{i+1} - c_i m_{i+1} \\
        t_{i+1} - c_{i+1} m_{i+1} & \geq t_i - c_{i+1} m_i \label{P3-IC2}
    \end{align}
    Summing up the two inequalities and rearranging the terms, we get
    \begin{align}
        (c_i - c_{i+1}) m_{i+1} & \geq (c_i - c_{i+1}) m_i \\
         m_{i+1} & \geq  m_i \label{data proposition}
    \end{align}
    since $c_i - c_{i+1} > 0$. \newline Now, we show $t_{i+1} \geq t_i$ using proof by contradiction. Suppose $t_{i+1} < t_i$. Then,
    \begin{equation}
        t_i - c_{i+1} m_i \geq t_i - c_{i+1} m_{i+1} > t_{i+1} - c_{i+1} m_{i+1},
    \end{equation}
    where the first inequality follows from (\ref{data proposition}). This contradicts 
(\ref{P3-IC2}).
\end{proof}
\begin{lemma}[Adjacent comparisons constraints] \label{ACCs}
    Under incomplete information, the IC constraints can be reduced to the following adjacent comparisons constraints:
    \begin{equation}
        \left\{\begin{array}{ll}
        t_i - c_i m_i \geq t_{i-1} - c_i m_{i-1} & \forall i \in \{2,\dots, I\} \\
        t_i - c_i m_i \geq  t_{i+1} - c_i m_{i+1} & \forall i \in \{1,\dots, I-1\}
        \end{array}\right. .
    \end{equation}
\end{lemma}
\begin{proof}
    First note that the necessity follows trivially from the IC constraints. Here we focus on proving the sufficiency. The main idea is to use the transitivity of the inequality and the fact that $c_i > c_{i+1}, \; \forall i \in \{1,\dots, I-1\}$. We proceed with the proof in two steps:
    \begin{enumerate}
        \item For $i \in \{2, \dots, I\}$, we show that satisfying $t_i - c_i m_i \geq t_{i-1} - c_i m_{i-1}$ implies $t_i - c_i m_i \geq t_{j} - c_i m_{j}, \; \forall j \leq i-1$.
        \item For $i \in \{1, \dots, I-1\}$, we show that satisfying $t_i - c_i m_i \geq t_{i+1} - c_i m_{i+1}$ implies $t_i - c_i m_i \geq t_{j} - c_i m_{j}, \; \forall j \geq i+1$.
    \end{enumerate}
    \underline{Step 1}: When $i=2$, the result follows directly from the definition of the downward adjacent comparisons constraints. Consider any $i \geq 3$. By definition,
    \begin{align}
        t_i - c_i m_i & \geq t_{i-1}) - c_i m_{i-1} \\
        t_{i-1} - c_{i-1} m_{i-1} & \geq t_{i-2} - c_{i-1} m_{i-2} \\
        & \cdots \notag \\
         t_{2} - c_{2} m_{2} & \geq t_1 - c_2 m_1
    \end{align}
    Rearrange the terms and we get:
    \begin{align}
        t_i - t_{i-1} & \geq  c_i (m_i - m_{i-1}) \label{dac 1}\\
        t_{i-1} -  t_{i-2} & \geq c_{i-1} (m_{i-1} - m_{i-2}) \geq c_i (m_{i-1} - m_{i-2}) \label{dac 2} \\
        & \cdots \notag \\
         t_{2} - t_1  & \geq c_2 (m_2 - m_1) \geq c_i (m_2 - m_1) \label{dac 3}
    \end{align}
    Summing up (\ref{dac 1}) and (\ref{dac 2}), we get
    \begin{equation}
        t_i -  t_{i-2} \geq c_i (m_{i-1} - m_{i-2}).
    \end{equation}
    Iteratively summing up the inequalities, we obtain $\forall j \leq i-1$,
    \begin{equation}
        t_i -  t_{j} \geq  c_i (m_i - m_{j}).
    \end{equation}
    Rearranging the terms gives the final result
    \begin{equation}
        t_i -  c_i m_i \geq   t_{j} - c_i m_{j}, \; \forall j \leq i-1.
    \end{equation}
    \underline{Step 2}: When $i=I-1$, the result follows directly from the definition of the upward adjacent comparisons constraints. Consider any $i \leq I-2$. By definition,
    \begin{align}
        t_i - c_i m_i & \geq t_{i+1} - c_i m_{i+1} \\
        t_{i+1} - c_{i+1} m_{i+1} & \geq t_{i+2} - c_{i+1} m_{i+2} \\
        & \cdots \notag \\
         t_{I-1} - c_{I-1} m_{I-1} & \geq t_I - c_{I-1} m_I
    \end{align}
     Rearrange the terms and we get:
    \begin{align}
         t_{i+1} - t_i  & \leq  c_i (m_{i+1} - m_i) \label{uac 1}\\
        t_{i+2} -  t_{i+1} & \leq c_{i+1} (m_{i+2} - m_{i+1}) \leq c_i (m_{i+2} - m_{i+1}) \label{uac 2} \\
        & \cdots \notag \\
         t_{I} - t_{I-1}  & \leq c_{I-1} (m_I - m_{I-1}) \leq c_i (m_I - m_{I-1}) \label{uac 3}
    \end{align}
    Summing up (\ref{uac 1}) and (\ref{uac 2}), we get
    \begin{equation}
        t_{i+2} -  t_i \leq c_i (m_{i+2} - m_i).
    \end{equation}
    Iteratively summing up the inequalities, we obtain $\forall j \geq i+1$,
    \begin{equation}
        t_j -  t_i \leq  c_i (m_j - m_i).
    \end{equation}
    Rearranging the terms gives the final result
    \begin{equation}
        t_i -  c_i m_i \geq   t_{j} - c_i m_{j}, \; \forall j \geq i+1.
    \end{equation}
\end{proof}

\begin{lemma}[Binding downward adjacent comparisons constraints]\label{binding DACs}
    With incomplete information and $m_{i+1} \geq m_i$ for all $i\in\{1,\dots, I-1\}$, the downward adjacent comparisons constraints must bind, and the upward adjacent comparisons constraints are redundant, i.e.,
    \begin{equation}
        t_i - c_i m_i = t_{i-1} - c_i m_{i-1}, \forall i \in \{2, \dots, I\}, \label{P7-R1}
    \end{equation}
    which implies
    \begin{equation}
        t_i - c_i m_i \geq t_{i+1} - c_i m_{i+1}, \forall i \in \{1, \dots, I-1\}. \label{P7-R2}
    \end{equation}
\end{lemma}
\begin{proof}
    We first prove that under the premise of binding downward adjacent comparisons constraints, the upward adjacent comparisons constraints hold true. $\forall i \in \{2, \dots, I\}$,
    \begin{align}
        t_i - c_i m_i & = t_{i-1} - c_i m_{i-1}, \\
        t_i -  t_{i-1}  & = c_i (m_i - m_{i-1}) \leq c_{i-1} (m_i - m_{i-1}),  
    \end{align}
    where the last inequality follows from $c_i < c_{i-1}$. Rearranging the terms, we get
    \begin{align}
         t_{i-1} - c_{i-1} m_{i-1} \geq t_i -  c_{i-1} m_i, \forall  i \in \{2, \dots, I\}.
    \end{align}
    Rewriting the indices gives
    \begin{equation}
        t_{i} - c_{i} m_{i} \geq t_{i+1} -  c_{i} m_{i+1}, \forall  i \in \{1, \dots, I-1\}.
    \end{equation}
    \medskip
    
    We then prove that the downward adjacent comparisons constraints bind by induction, where in each step we utilize proof by contradiction. The key idea is to \textit{construct an improved solution that demonstrates that a solution wherein some downward adjacent comparisons constraints are loose cannot be optimal}. \\
    \medskip
    
    \noindent \underline{Base case}: Assume there exists a solution ${\{(t_i, m_i)\}}_{i=1}^I$ to the constrained optimisation problem such that
    \begin{equation}
        t_I - c_I m_I > t_{I-1}- c_I m_{I-1}. \label{P7-Step1-0}
    \end{equation}
    \underline{Claim}: \textit{In this case, it must be that $m_I > m_{I-1}$.} \\ Assume not, i.e., $m_I = m_{I-1}$. Then, (\ref{P7-Step1-0}) implies $t_I > t_{I-1}$. But this results in
    \begin{equation}
        t_{I-1} - c_{I-1} m_{I-1} < t_I - c_{I-1} m_I,
    \end{equation}
    which violates the upward adjacent comparison constraint of type-$(I-1)$ parties. Therefore, $m_I > m_{I-1}$. 

    Construct $(\hat{t}_{I-1}, \hat{m}_{I-1})$ such that
    \begin{align}
        \hat{t}_{I-1} - c_{I-1} \hat{m}_{I-1} = & t_{I-1} - c_{I-1} m_{I-1}, \\
        \hat{t}_{I-1} - c_{I} \hat{m}_{I-1} = & t_{I} - c_{I} m_{I}.
    \end{align}
    Solving the above system of equations, we get
    \begin{align}
        \hat{m}_{I-1} = & \frac{1}{c_{I-1}- c_I} \left[(t_I - c_I m_I) - (t_{I-1}-c_{I-1}m_{I-1})\right] \label{P7-S1-C1},\\
        \hat{t}_{I-1} =& \frac{1}{c_{I-1} - c_I} \left[c_{I-1}(t_I- c_I m_I) - c_I (t_{I-1} - c_{I-1} m_{I-1})\right] \label{P7-S1-C2}.
    \end{align}
    Note that (\ref{P7-S1-C1}) and (\ref{P7-S1-C2}) imply that $\hat{m}_{I-1} > m_{I-1}$ and $\hat{t}_{I-1}> t_{I-1}$. To see this,
    \begin{align}
        \hat{m}_{I-1} - {m}_{I-1} = \frac{1}{c_{I-1}- c_I} \left[(t_I - c_I m_I) - (t_{I-1}-c_{I}m_{I-1})\right] > 0  \label{P7-S1-C3},\\
        \hat{t}_{I-1} - t_{I-1} = \frac{c_{I-1}}{c_{I-1} - c_I} \left[(t_I- c_I m_I) - (t_{I-1} - c_I m_{I-1})\right]  > 0 \label{P7-S1-C4} .
    \end{align}
    The last inequalities are given by the assumption (\ref{P7-Step1-0}). Similarly, we can show $\hat{m}_{I-1} \leq m_I$ and $\hat{t}_{I-1} \leq t_I$, using the upward adjacent comparison constraint of type-$(I-1)$ parties.

    By construct, all constraints (IC, UAC, DAC \& IR) are still satisfied for parties of type $(I-1)$. To show that $\{(\hat{t}_{I-1}, \hat{m}_{I-1})\} \cup \{(t_i, m_i)\}_{i \neq I-1}$ satisfies all constraints, we only have to consider the constraints for type-$(I-2)$ parties. Specifically, we are left to show that the UAC constraint still holds for parties of type $(I-2)$.

    Multiply both sides of (\ref{P7-S1-C3}) by $c_{I-2}$. and subtract it from (\ref{P7-S1-C4}) and we get
    \begin{equation}
        \hat{t}_{I-1} - t_{I-1} - c_{I-2} (\hat{m}_{I-1} - {m}_{I-1}) < 0.
    \end{equation}
    The last inequality is given by $c_{I-2} > c_{I-1} $. Rearranging the terms and combining it with the original UAC constraint for type-$(I-2)$ parties, we obtain
    \begin{equation}
        \hat{t}_{I-1} - c_{I-2} \hat{m}_{I-1} < t_{I-1} - c_{I-2} m_{I-1} \leq t_{I-2} - c_{I-2} m_{I-2}.
    \end{equation}
    Thus, the UAC constraint for type-$(I-2)$ parties is still satisfied. But since $\sum_{i\neq I-1} n_i m_i + n_{I-1} \hat{m}_{I-1} > \sum_{i=1}^I n_i m_i$, ${\{(r_i, m_i)\}}_{i=1}^I$ cannot be a solution to the constrained optimisation problem. Therefore, it must be that \begin{equation}
        t_I - c_I m_I = t_{I-1}- c_I m_{I-1}.
    \end{equation}

    \bigskip
    \noindent \underline{Induction step}: Assume that the DAC constraints bind for all $k = i+1, \dots, I$ private types. Consider the case for $k=i$ where $i\geq 2$, and assume that there exists a solution ${\{(t_k, m_k)\}}_{k=1}^I$ to the constrained optimisation problem such that the DAC does not bind for type-$i$ parties, i.e., 
    \begin{equation}
        t_i - c_i m_i > t_{i-1}- c_i m_{i-1} \label{P7-Step1-1},
    \end{equation}
    from which we observe that $m_i > m_{i-1}$. Construct $(\hat{t}_{i-1}, \hat{m}_{i-1})$ such that 
    \begin{align}
        \hat{t}_{i-1} - c_{i-1} \hat{m}_{i-1} = & t_{i-1} - c_{i-1} m_{i-1}, \\
        \hat{t}_{i-1} - c_{i} \hat{m}_{i-1} = & t_{i} - c_{i} m_{i}.
    \end{align}
    Solving the above system of equations, we get
    \begin{align}
        \hat{m}_{i-1} = & \frac{1}{c_{i-1}- c_i} \left[(t_i - c_i m_i) - (t_{i-1}-c_{i-1}m_{i-1})\right] \label{P7-S2-C1},\\
        \hat{t}_{i-1} =& \frac{1}{c_{i-1} - c_i} \left[c_{i-1}(t_i- c_i m_i) - c_i (t_{i-1} - c_{i-1} m_{i-1})\right] \label{P7-S2-C2},
    \end{align}
    from which we derive $\hat{m}_{i-1} > m_{i-1}$ and $\hat{t}_{i-1}> t_{i-1}$. To see this,
    \begin{align}
        \hat{m}_{i-1} - {m}_{i-1} = \frac{1}{c_{i-1}- c_i} \left[(t_i - c_i m_i) - (t_{i-1}-c_{i}m_{i-1})\right] > 0,  \label{P7-S2-C3}\\
        \hat{t}_{i-1} - t_{i-1} = \frac{c_{i-1}}{c_{i-1} - c_i} \left[(t_i- c_i m_i) - (t_{i-1} - c_i m_{i-1})\right]  > 0 \label{P7-S2-C4} .
    \end{align}
    Using the same trick, we drive $\hat{m}_{i-1} \leq m_i$ and $\hat{t}_{i-1} \leq t_i$. 
    
    For $i > 2$, multiply both sides of (\ref{P7-S2-C3}) by $c_{i-2}$. and subtract it from (\ref{P7-S2-C4}) and we get
    \begin{equation}
        \hat{t}_{i-1} - t_{i-1} - c_{i-2} (\hat{m}_{i-1} - {m}_{i-1}) < 0.
    \end{equation}
    The last inequality is given by $c_{i-2} > c_{i-1} $. When $i = 2$, this step is trivial (as type-1 is the lowest possible type). As a result, all constraints are satisfied by the constructed solution $\{(\hat{t}_{i-1}, \hat{m}_{i-1})\} \cup \{(t_k, m_k)\}_{k \neq i-1}$. As $\sum_{k\neq i-1} n_k m_k + n_{i-1} \hat{m}_{i-1} > \sum_{k=1}^I n_k m_k$, ${\{(t_k, m_k)\}}_{k=1}^I$ cannot be a solution to the constrained optimisation problem. Therefore, it must be that
    \begin{equation}
         t_i - c_i m_i = t_{i-1}- c_i m_{i-1}.
    \end{equation}

    \bigskip
    By induction, the DAC constraints must bind for $i \in \{2, \dots, I\}$.
    \end{proof}

\setcounter{theorem}{0}
\begin{theorem}
\label{equivalent constraints theorem appendix}
    With incomplete information, the following two sets of constraints are equivalent.
    \begin{enumerate}[label=(\alph*)] 
      \item $t_i - c_i m_i \geq t_j - c_j m_j, \forall i, j;$
      \item $\left\{\begin{array}{l}
      m_{i+1} \geq m_{i}, \forall i \in\{1, \dots, I-1\};\\
      t_i - c_i m_i = t_{i-1} - c_i m_{i-1}, \forall i \in \{2, \dots, I\}.
      \end{array}\right.$
    \end{enumerate}
\end{theorem}
\begin{proof}
    The proof of the direction from $(a)$ to $(b)$ follows readily from Proposition \ref{proportional fairness} and Lemma \ref{binding DACs}. We are left to show that $(b)$ implies $(a)$. As in Lemma \ref{ACCs} we have already shown that the IC constraints are equivalent to the adjacent comparisons constraints, we only need to show that $(b)$ implies the adjacent comparisons constraints. This is again given by Lemma \ref{binding DACs}.
\end{proof}

\begin{proposition}[Weak efficiency]
    With incomplete information, the most cost-efficient type (parties with the lowest variable cost) must be offered the best model, i.e., \begin{equation}
        t_I = \mathbb{E}_{n_I \geq 1}\left[v\left(a\left(\sum_{i=1}^I n_i m_i\right)\right)\right].
    \end{equation}
\end{proposition}
\begin{proof}
     For this proof, we focus on the following constrained optimisation problem after applying Lemma \ref{ACCs}:
    \begin{align}
        & \max_{{\{t_i, m_i\}}_{i=1}^I} \mathbb{E}_{n\sim \mathrm{Mul}(N,p)}\left[a\left(\sum_{i=1}^I n_i m_i\right)\right] \\
        \text{s.t.   } & \left\{\begin{array}{l}
        t_i-c_i m_i - f_i \geq 0, \; \forall i; \\
        t_i-c_i m_i \geq t_{i-1}-c_i m_{i-1}, \; \forall i \in \{2, \dots, I\}; \\
        t_i-c_i m_i \geq t_{i+1}-c_i m_{i+1}, \; \forall i \in \{1, \dots, I-1\};\\
        t_i \leq \mathbb{E}_{n_i \geq 1}\left[v\left(a\left(\sum_{i=1}^I n_i m_i\right)\right)\right], \; \forall i.
        \end{array}\right.
    \end{align}
    For simplicity, we omit the subscript of the expectation operator hereafter. Unless explicitly stated otherwise, $\mathbb{E}[\cdot] \triangleq \mathbb{E}_{n\sim \mathrm{Mul}(N,p)}[\cdot]$. The Lagrangian is
    \begin{align}
        \mathcal{L} =& \mathbb{E}\left[a\left(\sum_{i=1}^I n_i m_i\right)\right] + \sum_{i=1}^I \lambda_i ( t_i - c_i m_i - f_i)- \sum_{i=1}^I \mu_i \left\{t_i - \mathbb{E}_{n_i \geq 1}\left[v\left(a\left(\sum_{i=1}^I n_i m_i\right)\right)\right]\right\} \notag \\ &+ \sum_{i=2}^{I} \alpha_i \left\{t_i - c_i m_i - t_{i-1}+ c_i m_{i-1}\right\} + \sum_{i=1}^{I-1} \beta_i \left\{t_i - c_i m_i - t_{i+1} + c_i m_{i+1}\right\}. 
    \end{align}
    \textbf{Stationarity} ($2 \leq i \leq I-1$)
    \begin{align}
        \frac{\partial L}{\partial m_i} =& \mathbb{E}\left[ a^{\prime}\left(\sum_{i=1}^I n_i m_i\right)n_i\right] - \lambda_i c_i - \alpha_i c_i + \alpha_{i+1}  c_{i+1} \notag \\ &- \beta_i c_i + \beta_{i-1} c_{i-1} + \sum_{k=1}^I\mu_k \mathbb{E}_{n_k \geq 1}\left[v^{\prime}\left(a\left(\sum_{i=1}^I n_i m_i\right)\right)a^{\prime}\left(\sum_{i=1}^I n_i m_i\right) n_i\right] =0. \label{II-G-S1}\\
         \frac{\partial L}{\partial t_i} =& \lambda_i + \alpha_i - \alpha_{i+1} + \beta_i - \beta_{i-1} - \mu_i = 0.
    \end{align}
    \textbf{Stationarity} ($i=1$)
    \begin{align}
        \frac{\partial L}{\partial m_1} =& \mathbb{E}\left[a^{\prime}\left(\sum_{i=1}^I n_i m_i\right)n_1\right] - \lambda_1 c_1 + \alpha_{2}  c_{2} - \beta_1 c_1 + \sum_{k=1}^I\mu_k \mathbb{E}_{n_k \geq 1}\left[v^{\prime}\left(a\left(\sum_{i=1}^I n_i m_i\right)\right)a^{\prime}\left(\sum_{i=1}^I n_i m_i\right)n_1\right]  =0.  \label{II-E1-S1}\\
         \frac{\partial L}{\partial t_1} =& \lambda_1 - \alpha_{2} + \beta_1 - \mu_1 = 0.
    \end{align}
     \textbf{Stationarity} ($i=I$)
    \begin{align}
        \frac{\partial L}{\partial m_I} =& \mathbb{E}\left[ a^{\prime}\left(\sum_{i=1}^I n_i m_i\right)n_I\right] - \lambda_I c_I - \alpha_I c_I + \beta_{I-1} c_{I-1} \notag \\ &+ \sum_{k=1}^I\mu_k \mathbb{E}_{n_k \geq 1}\left[v^{\prime}\left(a\left(\sum_{i=1}^I n_i m_i\right)\right)a^{\prime}\left(\sum_{i=1}^I n_i m_i\right) n_I \right]=0. \label{II-E2-S1}\\
         \frac{\partial L}{\partial t_I} =& \lambda_I + \alpha_I - \beta_{I-1} - \mu_I = 0. \label{II-E2-S2}
    \end{align}
    \textbf{Primal feasibility}
    \begin{align}
        &t_i-c_i m_i - f_i \geq 0, \; \forall i. \label{II-PF1}\\
        &t_i-c_i m_i \geq t_{i-1}-c_i m_{i-1}, \; \forall i \in \{2, \dots, I\}. \label{II-PF2}\\
        &t_i-c_i m_i \geq t_{i+1}-c_i m_{i+1}, \; \forall i \in \{1, \dots, I-1\}. \label{II-PF3}\\
        &t_i \leq \mathbb{E}_{n_i \geq 1}\left[v\left(a\left(\sum_{i=1}^I n_i m_i\right)\right)\right], \; \forall i. \label{II-PF4}
    \end{align}
    \textbf{Dual feasibility}
    \begin{align}
        &\lambda_i \geq 0, \; \forall i. \\
        &\alpha_i \geq 0, \; \forall i \in \{2, \dots, I\}. \\
        &\beta_i \geq 0, \; \forall i \in \{1, \dots, I-1\}.\\
        &\mu_i \geq 0, \; \forall i.
    \end{align}
    \textbf{Complementary slackness}
    \begin{align}
        &\lambda_i( t_i- c_i m_i - f_i) = 0, \; \forall i. \label{II-CS1}\\
        &\alpha_i \left\{t_i - c_i m_i - t_{i-1}+ c_i m_{i-1}\right\} = 0, \; \forall i \in \{2, \dots, I\}. \label{II-CS2}\\
        &\beta_i \left\{t_i - c_i m_i - t_{i+1} + c_i m_{i+1}\right\} = 0, \; \forall i \in \{1, \dots, I-1\}. \label{II-CS3}\\
        &\mu_i \left\{t_i - \mathbb{E}_{n_i \geq 1}\left[v\left(a\left(\sum_{i=1}^I n_i m_i\right)\right)\right]\right\} = 0, \; \forall i. \label{II-CS4}
    \end{align} 
    Now we show that (\ref{II-PF4}) must bind for $i=I$. Assume not. Then, by (\ref{II-CS4}), we have $\mu_I = 0$. (\ref{II-E2-S2}) implies
    \begin{equation}
        \lambda_I = -\alpha_I + \beta_{I-1}.
    \end{equation}
    Plug it into (\ref{II-E2-S1}) and we get
    \begin{equation}
    \text{LHS} = \mathbb{E}\left[a^{\prime}\left(\sum_{i=1}^I n_i m_i\right)n_I\right] + \beta_{I-1} (c_{I-1}- c_{I}) + \sum_{k=1}^{I-1}\mu_k \mathbb{E}_{n_k \geq 1}\left[v^{\prime}\left(a\left(\sum_{i=1}^I n_i m_i\right)\right)a^{\prime}\left(\sum_{i=1}^I n_i m_i\right)n_I\right].
    \end{equation}
    The first term is strictly positive, and the second and third terms are non-negative. Thus,
    $\text{LHS} > 0,$
    which yields a contradiction. Therefore, it must be that \begin{equation}
        t_I = \mathbb{E}_{n_I \geq 1}\left[v\left(a\left(\sum_{i=1}^I n_i m_i\right)\right)\right].
    \end{equation}
\end{proof}
\begin{corollary}
    With incomplete information, the IR constraint must be binding for at least one type of parties, i.e.,
    \begin{equation}
        \exists i \text{ s.t. } t_i - c_i m_i - f_i = 0.
    \end{equation}
\end{corollary}
\begin{proof}
    For (\ref{II-G-S1}), (\ref{II-E1-S1}) and (\ref{II-E2-S1}) to hold, the necessary conditions are
    \begin{equation}
        \left\{\begin{array}{ll}
        \lambda_1 c_1 - \alpha_2 c_2 + \beta_1 c_1 >0     &  \\
        \lambda_i c_i + \alpha_i c_i - \alpha_{i+1} c_{i+1} +\beta_i c_i - \beta_{i-1} c_{i-1} > 0    & \forall i \in \{2, \dots, I-1\}\\
        \lambda_I c_I + \alpha_I c_I - \beta_{I-1} c_{I-1} >0
        \end{array}\right. .
    \end{equation}
    Summing up the inequalities, we get
    \begin{equation}
        \sum_{i=1}^I \lambda_i c_i > 0.
    \end{equation}
    Therefore, $\exists i$ such that $\lambda_i > 0$ (which can be shown using proof by contradiction).
\end{proof}

Now we focus on proving Proposition \ref{break even condtion}. We first establish six lemmas that are crucial to the proof. 
\begin{lemma}
    \label{qualifying contract solutions}
    Suppose that ${\{(t_i , m_i)_{i=1}^I\}}$ is a solution to the constrained optimization problem. For $j \in \{2, \dots, I\}$, if $t_j - c_1 m_j > f_1$, then $t_k - c_1 m_k > f_1, \forall k < j.$
\end{lemma}
\begin{proof}
    From the binding DAC constraints, we have
    \begin{equation}
        t_{j-1} - c_j m_{j-1} = t_j - c_j m_j.
    \end{equation}
    Then,
    \begin{align}
        & t_{j-1} - c_1 m_{j-1} +  c_1 m_{j-1} - c_j m_{j-1} = t_j -  c_1 m_j +  c_1 m_j - c_j m_j, \\
        & t_{j-1} - c_1 m_{j-1} + (c_1 - c_j) (m_{j-1} - m_j) = t_j -  c_1 m_j > f_1.
    \end{align}
    As $(c_1 - c_j) (m_{j-1} - m_j) \leq 0$, we have
    \begin{equation}
        t_{j-1} - c_1 m_{j-1} \geq t_{j-1} - c_1 m_{j-1} + (c_1 - c_j) (m_{j-1} - m_j) = t_j -  c_1 m_j > f_1.
    \end{equation}    
    Applying the above result iteratively, we get
    \begin{equation}
        t_k - c_1 m_k > f_1.
    \end{equation}    
\end{proof}

\begin{lemma} \label{Jensen's inequality}
For $i \in \{1, \dots, I-1\}$,
    \begin{equation}
        \frac{c_1 - c_i}{c_1 - c_{i+1}} f_{i+1} + \frac{c_i}{c_{i+1}} f_1 \geq f_i.
    \end{equation}
\end{lemma}
\begin{proof}
    Recall that
    \begin{equation}
        f_i = \tilde{u}_i(\bar{m}_i) = \tilde{v}(\bar{m}_i) - c_i \bar{m}_i.
    \end{equation}
    By treating $f_i$ as a function of $c_i$, we can write
    \begin{equation}
        f(c) = \tilde{v}(\bar{m}) - c \bar{m},
    \end{equation}
    where $\bar{m}$ is an implicit function of $c$. 
    
    The first derivative of $f$ is
    \begin{equation}
        \frac{df}{dc} = \tilde{v}^\prime(\bar{m}) \frac{d \bar{m}}{d c} - \bar{m} - c \frac{d \bar{m}}{d c} = - \bar{m},
    \end{equation}
    where the last equality follows after applying $\tilde{v}^\prime(\bar{m})=c$ from the proof of Proposition \ref{reservation utility}.

    The second derivative of $f$ is
    \begin{equation}
        \frac{d^2 f}{{(dc)}^2} = - \frac{d \bar{m}}{d c} \geq 0,
    \end{equation}
    where the last inequality follows from $\frac{d \bar{m}}{d c} \leq 0$ implied by Proposition \ref{reservation utility}. Therefore, $f(\cdot)$ is a convex function.

    For $i \in \{1, \dots, I-1\}$, by Jensen's inequality,
    \begin{equation}
        \frac{c_1 - c_i}{c_1 - c_{i+1}} f(c_{i+1}) + \frac{c_i}{c_{i+1}} f(c_1) \geq f(c_i).
    \end{equation}
    By definition,
    \begin{equation}
        \frac{c_1 - c_i}{c_1 - c_{i+1}} f_{i+1} + \frac{c_i}{c_{i+1}} f_1 \geq f_i.
    \end{equation}
\end{proof}

\begin{lemma}
    \label{ruling out contract solutions 1}
    Let $K$ denote the threshold type of parties for which
    \begin{align}
        t_i - c_1 m_i \leq & f_1, \;\;\; \forall i > K, \\
        t_i - c_1 m_i > & f_1 \;\;\; \forall i \leq K.
    \end{align}
    Then, for $K\in\{1, \dots, I-1\}$, ${\{(t_i, m_i)_{i=1}^I\}}$ cannot be a solution to the constrained optimization problem.
\end{lemma}
\begin{proof}
    We prove the lemma by contradiction. Given $K \in \{1, \dots, I-1\}$, assume that ${\{(t_i, m_i)_{i=1}^I\}}$ is a solution to the constrained optimization problem with
    \begin{align}
        t_i - c_1 m_i \leq & f_1, \;\;\; \forall i > K, \\
        t_i - c_1 m_i > & f_1 \;\;\; \forall i \leq K.
    \end{align}

    We construct the alternative contract ${\{(\hat{t}_i, \hat{m}_i)_{i=1}^{K}\}}$ that satisfies
    \begin{align}
        \left\{\begin{array}{l}
        \hat{t}_i - c_1 \hat{m}_i = f_1, \forall i \leq K, \\
        \hat{t}_K - c_{K+1} \hat{m}_K = t_{K+1} - c_{K+1} m_{K+1}, \\
        \hat{t}_i - c_{i+1} \hat{m}_i = \hat{t}_{i+1} - c_{i+1} \hat{m}_{i+1}, \forall i \leq K-1.
        \end{array}\right.
    \end{align}
    Solving the system of equations, we obtain
    \begin{align}
        \hat{m}_K & = \frac{1}{c_1 - c_{K+1}} \left[t_{K+1} - c_{K+1} m_{K+1} - f_1\right] \\
        \hat{m}_i & = \frac{1}{c_1 - c_{i+1}} \left[\hat{t}_{i+1} - c_{i+1} \hat{m}_{i+1} - f_1\right], \forall i \leq K-1, \\
        \hat{t}_K & = \frac{c_1}{c_1 - c_{K+1}} \left[t_{K+1} - c_{K+1} m_{K+1}\right] - \frac{c_{K+1}}{c_1 - c_{K+1}} f_1 \\
        \hat{t}_i & = \frac{c_1}{c_1 - c_{K+1}} \left[\hat{t}_{i+1} - c_{i+1} \hat{m}_{i+1}\right] - \frac{c_{i+1}}{c_1 - c_{i+1}} f_1, \forall i \leq K-1.
    \end{align}
    
    \noindent \underline{Claim 1}: \textit{For $i \in \{1, \dots, K\}$, $\hat{t}_i - c_i \hat{m}_i \geq f_i$}.

    \noindent We prove the claim by induction.

    \medskip
    
    \noindent Consider first when $i = K$,
    \begin{align}
        \hat{t}_K - c_K \hat{m}_K = \; & \frac{c_1}{c_1 - c_{K+1}} \left[t_{K+1} - c_{K+1} m_{K+1}\right] - \frac{c_{K+1}}{c_1 - c_{K+1}} f_1 \\
        & - \frac{c_K}{c_1 - c_{K+1}} \left[t_{K+1} - c_{K+1} m_{K+1} - f_1\right] \\
        = \; & \frac{c_1 - c_K}{c_1 - c_{K+1}} \left[t_{K+1} - c_{K+1} m_{K+1}\right] + \frac{c_K - c_{K+1}}{c_1 - c_{K+1}} f_1 \\
        \geq \; & \frac{c_1 - c_K}{c_1 - c_{K+1}} f_{K+1} + \frac{c_K - c_{K+1}}{c_1 - c_{K+1}} f_1 \\
        \geq \; & f_K.
    \end{align}
    The last inequality follows from Lemma \ref{Jensen's inequality}.

    Now assume that $\hat{t}_{k+1} - c_{k+1} \hat{m}_{k+1} \geq f_{k+1}$. Consider the case when $i = k$:
    \begin{align}
        \hat{t}_{k} - c_k \hat{m}_{k} = \; & \frac{c_1}{c_1 - c_{k+1}} \left[\hat{t}_{k+1} - c_{k+1} \hat{m}_{k+1}\right] - \frac{c_{k+1}}{c_1 - c_{k+1}} f_1 \\
        & - \frac{c_k}{c_1 - c_{k+1}} \left[\hat{t}_{k+1} - c_{k+1} \hat{m}_{k+1} - f_1\right] \\
        = \; & \frac{c_1 - c_k}{c_1 - c_{k+1}} \left[\hat{t}_{k+1} - c_{k+1} \hat{m}_{k+1}\right] + \frac{c_k - c_{k+1}}{c_1 - c_{k+1}} f_1 \\
        \geq \; & \frac{c_1 - c_k}{c_1 - c_{k+1}} f_{k+1} + \frac{c_k - c_{k+1}}{c_1 - c_{k+1}} f_1 \\
        \geq \; & f_k.
    \end{align}
    The second last inequality uses the induction assumption, and the last inequality follows from Lemma \ref{Jensen's inequality}.

    \medskip
    
    \noindent \underline{Claim 2}: \textit{$\hat{m}_K \leq m_{K+1}$. In addition, for $i \in \{1, \dots, K-1\}$, $\hat{m}_i = \hat{m}_{i+1}$, which in words implies that the newly constructed alternative contract is a pooling contract.}

    \noindent To prove that $\hat{m}_K \leq m_{K+1}$, we do the subtraction:
    \begin{equation}
        \hat{m}_K - m_{K+1} = \frac{1}{c_1 - c_{K+1}} \left[t_{K+1} - c_1 m_{K+1} - f_1\right] \leq 0,
    \end{equation}
    where the last inequality follows from the assumption that
    \begin{equation}
        t_i - c_1 m_i \leq f_1, \;\;\; \forall i > K.
    \end{equation}
    Similarly, for $i \in \{1, \dots, K-1\}$,
    \begin{equation}
        \hat{m}_i - \hat{m}_{i+1} = \frac{1}{c_1 - c_{i+1}} \left[\hat{t}_{i+1} - c_1 \hat{m}_{i+1} - f_1\right] = 0,
    \end{equation}
    where the last equality follows by construction.

    \medskip

    \noindent \underline{Claim 3}: \textit{For $i \in \{1, \dots, K\}$, $\hat{m}_i > m_{K}$.}

    \noindent From Claim 2, we know that $\hat{m}_i = \hat{m}_{i+1}, \forall i \in\{1, \dots, K-1\}$. Therefore, $\hat{m}_i = \hat{m}_K$ for all $i \in \{1, \dots, K-1\}$. We only need to show that $\hat{m}_K \geq m_K$.
    \begin{align}
        \hat{m}_K - m_K = \; & \frac{1}{c_1 - c_{K+1}} \left[t_{K+1} - c_{K+1} m_{K+1} - f_1\right] - m_K \\
        = \; & \frac{1}{c_1 - c_{K+1}} \left[t_{K} - c_{K+1} m_K - f_1\right] - m_K \\
        = \; & \frac{1}{c_1 - c_{K+1}} \left[t_{K} - c_1 m_{K} - f_1\right] \\
        > \; & 0.
    \end{align}
    The second line follows from the binding downward adjacent comparison condition for the original set of contracts, i.e.,
    \begin{equation}
       t_K - c_{K+1} m_K =  t_{K+1} - c_{K+1} m_{K+1}.
    \end{equation}
    The last inequality follows from the assumption that
    \begin{equation}
        t_{K} - c_1 m_{K} > f_1.
    \end{equation}

    \medskip

    \noindent \underline{Claim 4}: \textit{For $i \in \{1, \dots, K\}$, $\hat{t}_i \leq t_{K+1}$.}
    
    \noindent Again, due to the nature of pooling contracts, we only need to show that $\hat{t}_K \leq t_{K+1}$.
    \begin{align}
        \hat{t}_K - t_{K+1} = \; & \frac{c_1}{c_1 - c_{K+1}} \left[t_{K+1} - c_{K+1} m_{K+1}\right] - \frac{c_{K+1}}{c_1 - c_{K+1}} f_1 - t_{K+1} \\
        = \; & \frac{c_{K+1}}{c_1 - c_{K+1}} \left[t_{K+1} - c_1 m_{K+1} - f_1 \right] \\
        \leq \; & 0.
    \end{align}
    The last inequality following from the assumption that
    \begin{equation}
        t_i - c_1 m_i \leq f_1, \;\;\; \forall i > K.
    \end{equation}
    It follows that
    \begin{equation}
        \hat{t}_K \leq t_{K+1}.
    \end{equation}
    \bigskip

    \noindent Now consider the following alternative contract ${\{(\hat{t}_i, \hat{m}_i)_{i=1}^K\}} \cup {\{(t_i, m_i)_{i=K+1}^I\}}$. It is feasible as
    \begin{itemize}
        \item It satisfies the \textbf{budget constraints} (BC). For $i \in \{1, \dots, K\}$, $\hat{t}_i \leq t_{K+1} \leq \dots \leq t_{I} \leq \mathbb{E}\left[v\left(a\left(\sum_{i=1}^I n_i m_i\right)\right)\right]$.
        \item It satisfies the \textbf{individual rationality} (IR) conditions. For $i > K$, the IR conditions hold by assumption. For $i\leq K$, the IR conditions are shown to hold in Claim 1.
        \item It satisfies the \textbf{incentive compatibility} (IC) conditions. First, from Theorem \ref{equivalent constraints theorem appendix} the IC conditions are equivalent to binding DAC constraints plus the weakly increasing contributions constraints. For $i \leq K$, the DAC continue to bind by construction. The weakly increasing contributions constraints continue to hold due to Claim 2.
    \end{itemize}
    But from Claim 3, we know that for $i\in\{1, \dots, K\}$, $\hat{m}_i > m_K \geq m_i$, where the last inequality follows from the fact that the weakly increasing contributions constraints hold for the original set of contract options ${\{(t_i, m_i)\}}_{i=1}^I$. Then,
    \begin{equation}
        \mathbb{E}\left[a\left(\sum_{i=1}^K n_i \hat{m}_i + \sum_{i=K+1}^I n_i m_i \right)\right] > \mathbb{E}\left[a\left(\sum_{i=1}^K n_i m_i + \sum_{i=K+1}^I n_i m_i \right)\right] =\mathbb{E}\left[a\left(\sum_{i=1}^I n_i m_i \right)\right].
    \end{equation}
    This contradicts the assumption that ${\{(t_i, m_i)\}}_{i=1}^I$ is a solution to the constrained optimization problem.
\end{proof}

\begin{lemma}
    Consider the hypothetical contract ${\{(\hat{t}_i,\hat{m}_i)_{i=1}^{I-1}\}}$ such that $\forall i \in \{1, \dots, I-1\}$,
    \begin{align}
        \left\{\begin{array}{l}
            \hat{t}_i - c_i \hat{m}_i - f_i = 0, \\
            \hat{t}_i - c_I \hat{m}_i - f_I = 0.
        \end{array}\right.
    \end{align}
    We have $\hat{t}_1 \leq \dots \leq\hat{t}_{I-1}$ and $\hat{m}_1 \leq \dots \leq \hat{m}_{I-1}$.
\end{lemma}
\begin{proof}
    Solving the system of equations,
    \begin{align}
        \left\{\begin{array}{l}
            \hat{t}_i - c_i \hat{m}_i - f_i = 0, \\
            \hat{t}_i - c_I \hat{m}_i - f_I = 0,
        \end{array}\right.
    \end{align}
    we get
    \begin{align}
        \hat{m}_i = & \; \frac{f_I - f_i}{c_i - c_I}, \\
        \hat{t}_i = & \; c_I \hat{m}_i + f_I = \frac{c_i}{c_i - c_I}f_I  - \frac{c_I}{c_i - c_I}f_i. \label{hypothetical rewards}
    \end{align}
    Consider
    \begin{align}
        \frac{d \hat{m}_i}{d c_i} = & \; \frac{\bar{m}_i}{c_i - c_I} - \frac{1}{{(c_i-c_I)}^2} (f_I - f_i) \\
        = & \; \frac{1}{c_i - c_I}\left[(c_i - c_I) \bar{m}_i - f_I + f_i\right].
    \end{align}
    Our goal now is to determine the sign of $(c_i - c_I) \bar{m}_i - f_I + f_i$.

    \medskip
    
    \noindent By definition,
    \begin{equation}
        f_I = \tilde{v}(\bar{m}_I) - c_I \bar{m}_I = \max_{m\geq 0} \tilde{v}(m) - c_I m.
    \end{equation}
    It follows that
    \begin{align}
        & \tilde{v}(\bar{m}_I) - c_I \bar{m}_I \geq \tilde{v}(\bar{m}_i) - c_I \bar{m}_i \\
        & f_I \geq \tilde{v}(\bar{m}_i) - c_i \bar{m}_i + (c_i - c_I)\bar{m}_i \\
        & f_I \geq f_i + (c_i - c_I)\bar{m}_i \\
        & (c_i - c_I)\bar{m}_i - f_I + f_i \leq 0.
    \end{align}
    As $c_i - c_I > 0$, we conclude
    \begin{equation}
        \frac{d \hat{m}_i}{d c_i} \leq 0.
    \end{equation}
    Therefore, $\hat{m}_1 \leq \dots \leq \hat{m}_I$. By Equation (\ref{hypothetical rewards}), it follows that $\hat{t}_1 \leq \dots \leq \hat{t}_I$.
\end{proof}

\begin{corollary} \label{transitivity of hypothetical inequalities}
    If $t_I < \frac{c_1}{c_1 - c_I} f_I - \frac{c_I}{c_1 - c_I} f_1$, then
    \begin{equation}
        t_I < \frac{c_i}{c_i - c_I} f_I - \frac{c_I}{c_i - c_I} f_i,
    \end{equation}
    for $i = 2, \dots, I-1$.
\end{corollary}
\begin{proof}
    This follows from the fact that $v(\hat{r}_1) \leq v(\hat{r}_i)$ for $i = 2, \dots, I-1$.
\end{proof}
\begin{lemma}
    \label{best model condtion}
    Suppose ${\{(t_i, m_i)_{i=1}^I\}}$ is a solution to the constrained optimization problem. It must be that
    \begin{equation}
        t_I \geq \frac{c_1}{c_1 - c_I} f_I - \frac{c_I}{c_1 - c_I} f_1.
    \end{equation}
\end{lemma}
\begin{proof}
    We prove the lemma by contradiction. Assume that ${\{(t_i, m_i)_{i=1}^I\}}$ is a solution to the constrained optimization problem such that
    \begin{equation}
        t_I < \frac{c_1}{c_1 - c_I} f_I - \frac{c_I}{c_1 - c_I} f_1.
    \end{equation}

    \noindent \underline{Claim 1}: The only solution that exists under this setting is a pooling contract, i.e., $t_i = t_I, m_i = m_I$ for all $i \in \{1, \dots, I\}$.

    \noindent To show the claim, we employ proof by contradiction. Assume that there exists a set of partially separating contract options (i.e., there are at least two distinct contract options). Denote the contract option chosen by a type-$I$ party as $(t_I, m_I)$. We use $i$ to denote the highest party type that opts for a different contract option from type I. Denote the contract option chosen by a type-$i$ party as $(t_i, m_i)$.

    From Proposition \ref{proportional fairness}, we know that in the presence of separating contract options,
    \begin{equation}
        t_I > t_i, \;\;\;\;\;\; m_I > m_i.
    \end{equation}
    However, we now show that there exists an improved solution by pooling the contract options designed for type-$i$ and type-$I$ parties together.

    Consider the alternative contract option $(\check{t}_i, \check{m}_i)$ such that $\check{t}_i = t_I$ and $\check{m}_i = m_I$. In other words, consider giving type-$i$ parties the contract originally designed for type-$I$ parties. By the nature of pooling contracts, we know that the IC conditions continue to hold. As we do not alter the contract option designed for type-$I$ parties, the BC constraint continues to hold. We are only left to show that the IR condition holds for type-$i$ parties with the new contract option. From Corollary \ref{transitivity of hypothetical inequalities}, we know the assumption
    \begin{equation}
        t_I < \frac{c_1}{c_1 - c_I} f_I - \frac{c_I}{c_1 - c_I} f_1
    \end{equation}
    implies
    \begin{equation}
        t_I < \frac{c_i}{c_i - c_I} f_I - \frac{c_I}{c_i - c_I} f_i.
    \end{equation}
    Rearrange the inequality and we get
    \begin{align}
        & (c_i - c_I) t_I < c_1 f_I - c_I f_1 \\
        & c_i (t_I - f_I) < c_I (t_I - f_i) \\
        & c_i (t_I - c_I m_I - f_I) < c_I (t_I - c_i m_I - f_i).
    \end{align}
    Since $c_i > c_I > 0$ and $t_I - c_I m_I - f_I \geq 0$, we have
    \begin{equation}
        t_I - c_i m_I - f_i \geq 0
    \end{equation}
    as is desired. Therefore, the IR constraint continues to hold, and ${\{(r_k, m_k)_{k\neq i}\}} \cup \{(\check{r}_i, \check{m}_i)\}$ is feasible. However,
    \begin{equation}
        a(\sum_{k\neq i}{n_k m_k} + n_i \check{m}_i) > a(\sum_{K=1}^I{n_k m_k})
    \end{equation}
    for non-zero $n_i$. ${\{(t_i, m_i)_{i=1}^I\}}$ cannot be the optimizer of the constrained optimization problem.
\end{proof}

\begin{lemma}
    \label{ruling out contract solutions 2}
    There cannot be a solution ${\{(t_i, m_i)_{i=1}^I\}}$ to the constrained optimization problem such that
    \begin{equation}
        t_i - c_1 m_i > f_1, \forall i \in \mathcal{I}.
    \end{equation}
\end{lemma}
\begin{proof}
Assume the opposite is true, i.e., there exists a solution ${\{(t_i, m_i)_{i=1}^I\}}$ to the constrained optimization problem such that
    \begin{equation}
        t_i - c_1 m_i > f_1, \forall i \in \mathcal{I}.
    \end{equation}     
    Consider the alternative contract $\{(\hat{t}_i, \hat{m}_i)_{i=1}^{I}\}$ that satisfies
    \begin{align}
        \left\{\begin{array}{l}
        \hat{t}_i - c_1 \hat{m}_i = f_1, \forall i \in \{1, \dots, I \}, \\
        \hat{t}_I = t_I, \\
        \hat{t}_i - c_{i+1} \hat{m}_i = \hat{t}_{i+1} - c_{i+1} \hat{m}_{i+1}, \forall i \leq I-1.
        \end{array}\right.
    \end{align}
    Solving the system of equations, we have
    \begin{align}
        \hat{m}_I & = \frac{t_I -f_1}{c_1}  \\
        \hat{m}_i & = \frac{1}{c_1 - c_{i+1}} \left[\hat{t}_{i+1} - c_{i+1} \hat{m}_{i+1} - f_1\right], \forall i \leq I-1, \\
        \hat{t}_I & = t_I \\
        \hat{t}_i & = \frac{c_1}{c_1 - c_{K+1}} \left[\hat{t}_{i+1} - c_{i+1} \hat{m}_{i+1}\right] - \frac{c_{i+1}}{c_1 - c_{i+1}} f_1, \forall i \leq I-1.
    \end{align}

    \noindent \underline{Claim 1}: \textit{For $i \in \{1, \dots, I\}$, $\hat{t}_i - c_i \hat{m}_i \geq f_i$}.

    \noindent We prove the claim by induction.

    \medskip
    
    \noindent Consider first when $i = I$,
    \begin{align}
        \hat{t}_I - c_I \hat{m}_I = \; & \frac{c_1 - c_I}{c_1} t_I + \frac{c_I}{c_1} f_1 \\
        \geq \; & \frac{c_1 - c_I}{c_1} \left( \frac{c_1}{c_1 - c_I} f_I - \frac{c_I}{c_1 - c_I} f_1\right) + \frac{c_I}{c_1} f_1 \\
        = \; & f_I.
    \end{align}
    The second last inequality utilizes Lemma \ref{best model condtion}.

    Now assume that $\hat{t}_{k+1} - c_{k+1} \hat{m}_{k+1} \geq f_{k+1}$. Consider the case when $i = k$:
    \begin{align}
        \hat{t}_{k} - c_k \hat{m}_{k} = \; & \frac{c_1}{c_1 - c_{k+1}} \left[\hat{t}_{k+1} - c_{k+1} \hat{m}_{k+1}\right] - \frac{c_{k+1}}{c_1 - c_{k+1}} f_1 \\
        & - \frac{c_k}{c_1 - c_{k+1}} \left[\hat{t}_{k+1} - c_{k+1} \hat{m}_{k+1} - f_1\right] \\
        = \; & \frac{c_1 - c_k}{c_1 - c_{k+1}} \left[\hat{t}_{k+1} - c_{k+1} \hat{m}_{k+1}\right] + \frac{c_k - c_{k+1}}{c_1 - c_{k+1}} f_1 \\
        \geq \; & \frac{c_1 - c_k}{c_1 - c_{k+1}} f_{k+1} + \frac{c_k - c_{k+1}}{c_1 - c_{k+1}} f_1 \\
        \geq \; & f_k.
    \end{align}
    The second last inequality follows from the induction assumption, and the last inequality is a consequence of Lemma \ref{Jensen's inequality}.

    \medskip
    
    \noindent \underline{Claim 2}: The newly constructed alternative contract is a pooling contract.

    \noindent To prove the result, for $i \in \{1, \dots, I-1\}$, we subtract $\hat{m}_{i+1}$ from $\hat{m}_i$:
    \begin{equation}
        \hat{m}_i - \hat{m}_{i+1} = \frac{1}{c_1 - c_{i+1}} \left[\hat{t}_{i+1} - c_1 \hat{m}_{i+1} - f_1\right] = 0,
    \end{equation}
    where the last equality follows by construction.

    \medskip

    \noindent \underline{Claim 3}: \textit{For $i \in \{1, \dots, I\}$, $\hat{m}_i > m_{I}$.}

    \noindent Claim 2 tells us that $\hat{m}_i = \hat{m}_{i+1}, \forall i \in\{1, \dots, K-1\}$. As a result, it suffices to show that $\hat{m}_I \geq m_I$.
    \begin{align}
        \hat{m}_I - m_I = \; & \frac{t_I - c_1 m_I - f_1}{c_1} \\
        = \; & \frac{1}{c_1 - c_{K+1}} \left[t_{K} - c_{K+1} m_K - f_1\right] - m_K \\
        > \; & 0.
    \end{align}
    The last inequality follows from the assumption that
    \begin{equation}
        t_{I} - c_1 m_{I} > f_1.
    \end{equation}

    \medskip

    \noindent \underline{Claim 4}: \textit{For $i \in \{1, \dots, I\}$, $\hat{t}_i \leq t_{I}$.}
    
    \noindent Due to the nature of pooling contracts, we only need to show $\hat{t}_I \leq t_I$. This holds true readily, as $\hat{t}_I = t_I$ by construction.
    
    \bigskip

    \noindent Putting the claims together, we see that ${\{(\hat{t}_i, \hat{m}_i)_{i=1}^I\}}$ is feasible as
    \begin{itemize}
        \item It satisfies the \textbf{budget constraints} (BC). For $i \in \{1, \dots, I\}$, $\hat{t}_i \leq t_{I} \leq \mathbb{E}\left[v\left(a\left(\sum_{i=1}^I n_i m_i\right)\right)\right]$, given by Claim 4.
        \item It satisfies the \textbf{individual rationality} (IR) conditions, given by Claim 1.
        \item It satisfies the \textbf{incentive compatibility} (IC) conditions. First, from Theorem \ref{equivalent constraints theorem appendix} the IC conditions are equivalent to binding DAC constraints plus the weakly increasing contributions constraints. The DAC continue to bind for the alternative contract by construction. The weakly increasing contributions constraints hold due to Claim 2.
    \end{itemize}
    However, from Claim 3, we know that for $i\in\{1, \dots, I\}$, $\hat{m}_i > m_I \geq m_i$, where the last inequality holds as the weakly increasing contributions constraints hold for the original set of contract options ${\{(t_i, m_i)_{i=1}^I\}}$. Then,
    \begin{equation}
        \mathbb{E}\left[a\left(\sum_{i=1}^K n_i \hat{m}_i + \sum_{i=K+1}^I n_i m_i \right)\right] > \mathbb{E}\left[a\left(\sum_{i=1}^K n_i m_i + \sum_{i=K+1}^I n_i m_i \right)\right] =\mathbb{E}\left[a\left(\sum_{i=1}^I n_i m_i \right)\right].
    \end{equation}
    This creates a contradiction with the assumption that ${\{(t_i, m_i)_{i=1}^I\}}$ is a solution to the constrained optimization problem.
\end{proof}

\begin{corollary}
    A solution to the constrained optimization problem cannot exist with
    \begin{equation}
        t_i - c_1 m_i > f_1, \exists i \in \{1, \dots, I\}.
    \end{equation}
    In other words, if ${\{(t_i, m_i)_{i=1}^I\}}$ is a solution to the constrained optimization problem, it must satisfy
    \begin{equation}
        t_i - c_1 m_i \leq f_1, \forall i \in \{1, \dots, I\}.
    \end{equation}
\end{corollary}

\begin{proof}
    The corollary follows by combining the results of Lemma \ref{qualifying contract solutions}, Lemma \ref{ruling out contract solutions 1} and Lemma \ref{ruling out contract solutions 2}.
\end{proof}

\begin{proposition}[Highest-cost type break even condition]
    With incomplete information, type-$1$ parties would obtain utilities no greater than their reservation level if choosing options designed for the other types. In addition, the individual rationality constraint must bind for parties of type $1$. Namely,
        \begin{align}
            t_j - c_1 m_j \leq & f_1 , \forall j \in \{2, \dots, I\}, \\
            t_1 - c_1 m_1 = & f_1.
        \end{align}
\end{proposition}

\begin{proof}
    The proposition follows from the previous lemma, after noting the fact that
    \begin{equation}
        t_1 - c_1 m_1 \geq f_1,
    \end{equation}
    which is required by the IR condition.
\end{proof}

\onecolumn

\section{Supplementary Sections}

\subsection{Related Work}


In collaborative machine learning, the incentivization issue arises due to the presence of information asymmetry where participants possess extra information that cannot be obtained by the coordinator. There has been a plethora of studies in economics on information asymmetry under the banners of principal-agency problem and contract theory (see \citealt{mwg1995, laffont_theory_2002, bolton2004} for a comprehensive treatment of the subject). 
Existing works can be largely subsumed into two strands depending on the source of the private information—the first concerns the principal's inability to observe agents' specific actions (termed the \textit{moral hazard} problem), while the second deals with the principal's inability to observe agents' private features, such as ability, that do not change within the scope of the game (called the \textit{adverse selection} problem). 
The monopolistic screening setup we adopt in this paper falls into the second strand and models the game where the principal moves first by offering a set of contract options, from which the agents then choose. The monopolistic part refers to the assumption of only one principal in the market, thereby devoid of competition. 
Regardless of which type of hidden information is being tackled, works in economics share the descriptive nature, stressing the theoretical explainability rather than the practical implementability. This differs from the goal of contract design for collaborative machine learning, which is innately prescriptive and demands attention to practical considerations.



Prior to our work, there has been a line of research that alludes to monopolistic screening in addressing the incentive issue in collaborative machine learning \citep{kang_incentive_2019, ding_incentive_2020, karimireddy_mechanisms_2022, Liu2023}. Our work differs from most of the works in the sense of using models as the rewards for the contracts and taking into account the stochasticity of the rewards due to different type realizations. Drawing on the limitations of previous works, we strengthen the coherence of our setup by ensuring that the utility functions have valid economic and machine learning meanings, and that the specification of the reservation utility reflects the training technology utilized by the collaboration.

In this regard, the setup of our work is most related to that of \citet{karimireddy_mechanisms_2022}, which utilizes models as rewards to maximize the total amount of collected data, assuming that agents have different per-unit data collection costs. 
Different from their work, we assume that the goal of the principal is to maximize the accuracy of the collectively trained model rather than the total amount of collected data. 
While under complete information, the two goals are equivalent, they lead to different optimal contracts under incomplete information (cf. Appendix A). Our work also differs in the setup of the individual rationality and incentive compatibility constraints, which are more in line with the definitions in the economic literature. While \citet{karimireddy_mechanisms_2022} focuses extensively on the complete information case and derives one specific incentive mechanism that satisfies the optimality condition, we elaborate on the incomplete information case and provide a general solution framework that characterizes the optimal contribution-reward pairs that must be satisfied by all optimal mechanisms. The formulation of the optimal contracting problem through constrained optimization also allows practitioners to readily tweak the components based on their practical needs—resulting in increased flexibility of the framework.


\subsection{Enacting Additional Fairness Requirements through Contract Design} \label{enacting additional fairness requirements}

In Section \ref{Contracting with Observable Costs} of the paper, we state that a coordinator can enact additional fairness requirements by modifying the IR constraints of the optimal contracting problem. Here we provide a viable framework under complete information. Note that it also addresses the concern some readers may have—that by partaking in the CML scheme each party receives the same level of utility as they do by training a model on their own. The modification comes with a cost—by enable additional surpluses to be gained by the participating parties, the coordinator must relinquish the insistence on obtaining the highest accuracy of the trained model—this reflects a long-lasting theme in the economic field about trading off performance and welfare.

The modified framework is as follows:
\begin{align}
        \max_{\{r_i, m_i\}_{i=1}^{I}} & \; a\left(\sum_{i=1}^I n_i m_i\right) \\
    \text{s.t.} & \; \left\{\begin{array}{l}
       v(r_i) - c_i m_i - f_i \geq s_i,  \forall i; \\
       \\ r_i \leq  a\left(\sum_{i=1}^I n_i m_i\right), \forall i .
    \end{array}\right.
\end{align}
where $\{s_i\}_{i=1}^I$ are hyper-parameters set by the scheme designer. Ceteris paribus, having positive surpluses for scheme participants is akin to increasing the reservation utilities for them, thereby increasing their bargaining power and reducing the amount of contribution that can be asked from each participant by the coordinator. While we demonstrate the framework under complete information, the same trick applies to the incomplete information scenario.

\subsection{Collaboration Failure in the Presence of Hidden Information}
\begin{tcolorbox}[title=TL;DR]
\textbf{Q:} Without any incentive mechanism, what could happen in collaborative machine learning involving multiple participants with different private costs? 
\tcblower
\textbf{A:} In some cases, we can predict a \textbf{complete collaboration failure} in which no participants make any contribution. It is also possible that \textbf{no sensible equilibrium exists}—the result of the collaboration is not predictable.
\end{tcolorbox}

\citet{karimireddy_mechanisms_2022} demonstrate that catastrophic freeriding could occur in a collaborative machine learning scheme where participants have different but observable data collection costs. They show that in this case only participants with the lowest data collection cost choose to contribute a non-negative amount to the scheme, while the others receive the model reward without any contribution.

We extend the analysis to the general setting where data collection cost is not publicly observable and participants only know the probability distribution of the private costs. In this incomplete information game, a participant holds a belief about other participants' contributions, based on which they determine the amount they would contribute. As a solution concept, we extend the \textbf{perfect Bayesian equilibrium (PBE)} (cf. \citealp{mwg1995}) to the collaborative machine learning setting:
\begin{definition}
    A \textbf{perfect Bayesian equilibrium in collaborative machine learning} is a set of contributions and beliefs such that \begin{itemize}
        \item \textbf{(Rationality)} Given a participant's belief on other participants' contributions, the contribution chosen by the participant should maximize their expected utility.
        \item \textbf{(Consistency)} The beliefs held by the participants are correct in the sense that they align with the equilibrium contributions. 
        \item \textbf{(Reasonability)} The equilibrium contributions are reasonable in the sense that participants of the same private type make the same contribution.
        \end{itemize}
\end{definition}
With the above equilibrium concept, we can now establish several important results about unincentivized collaboration machine learning in the presence of hidden information.
\begin{proposition}
\label{non-positive PBE}
    In a collaborative machine learning scheme involving multiple participants with different privately observable costs, a perfect Bayesian equilibrium where some participants contribute positive amount of data does not exist.
\end{proposition}
\begin{proof}
    Assume it is not the case, i.e., there exist participants of private type $i$ who contribute $m_i^\circ > 0$ in equilibrium. For brevity of notation, we let $M_{-i}(n)\triangleq \sum_{k \neq i} n_k m_k^\circ$ denote the total contribution of participants of all other types given type realization $n$, and we denote $v(a(\cdot))$ by $\tilde{v}$. By the reasonability requirement of a perfect Bayesian equilibrium, $m_i^\circ$ is determined as follows:
    \begin{equation}
    m_i^\circ := \arg\max_{m_i}\mathbb{E}_{n_i \geq 1} \left[ \tilde{v}(n_i m_i + M_{-i}(n))\right] - c_i m_i
    \end{equation}
    Equivalently, by the first-order condition, we have
    \begin{equation}
        c_i = \mathbb{E}_{n_i \geq 1} \left[ \tilde{v}^\prime(n_i m_i^\circ + M_{-i}(n)) n_i\right] >  \mathbb{E}_{n_i \geq 1} \left[ \tilde{v}^\prime(n_i m_i^\circ + M_{-i}(n))\right], \label{reasonability requirement}
    \end{equation}
    where the last inequality follows from realizing $n_i > 1$ with positive probability.

    Now, the rationality requirement of a perfect Bayesian equilibrium mandates that all participants should not have any incentive to deviate to another contribution amount. Consider a participant of type $i$. They decide not to deviate if 
    \begin{equation}
    m_i^\circ = \arg\max_{m_i}\mathbb{E}_{n_i \geq 1} \left[ \tilde{v}(m_i + (n_i-1) m_i^\circ + M_{-i}(n))\right] - c_i m_i.
    \end{equation}
    The optimal contribution $m_i^*$ for the participant satisfies the following first-order condition,
    \begin{equation}
        c_i = \mathbb{E}_{n_i \geq 1} \left[ \tilde{v}^\prime(m_i^* + (n_i-1) m_i^\circ + M_{-i}(n))\right].
    \end{equation}
    Combining the result with (\ref{reasonability requirement}), we have
    \begin{equation}
        \mathbb{E}_{n_i \geq 1} \left[ \tilde{v}^\prime(m_i^* + (n_i-1) m_i^\circ + M_{-i}(n))\right] >  \mathbb{E}_{n_i \geq 1} \left[ \tilde{v}^\prime(n_i m_i^\circ + M_{-i}(n))\right].
    \end{equation}
    Consequently, $m_i^* < m_i^\circ$. Therefore, the participant would indeed deviate—a contradiction.
\end{proof}

Proposition \ref{non-positive PBE} leaves us with only 1 candidate perfect Bayesian equilibrium, where all participants make no contribution, i.e., $m_i^\circ = 0, \forall i \in \mathcal{I}$. It is indeed the only perfect Bayesian equilibrium when all participants have 0 reservation utilities.
\begin{proposition}
    \label{collaboration failure}
    In a collaborative machine learning scheme involving multiple participants with different privately observable costs, if all participants have zero reservation utilities, then the only perfect Bayesian equilibrium is a complete collaboration failure where all participants make no contribution. 
\end{proposition}
\begin{proof}
    If all participants have zero reservation utilities, it must be that
    \begin{equation}
        c_i \leq \tilde{v}^\prime(0), \forall i.
    \end{equation}
    All participants would not deviate because there does not exist a positive amount of contribution that can earn them a positive profit. In other words,
    \begin{equation}
         0 = \arg\max_{m_i\geq 0} \tilde{v}(m_i) - c_i m_i.
    \end{equation}
\end{proof}

Finally, if some participants have non-zero reservation utilities, we establish the non-existence of a perfect Bayesian equilibrium, which gives us no predictive results about the collaboration.
\begin{proposition}
    \label{non-existence of PBE}
    In a collaborative machine learning scheme involving multiple participants with different privately observable costs, if some participants have non-zero reservation utilities, then no perfect Bayesian equilibrium exists. 
\end{proposition}
\begin{proof}
    Let $i$ denote one type of participants who have non-zero reservation utilities. It follows that
    \begin{equation}
        c_i > \tilde{v}^\prime(0).
    \end{equation}
    When all other participants make no contribution, a type-$i$ participant would want to deviate by contributing their reservation level of contribution $\tilde{m}_i >0$, as
    \begin{equation}
        \tilde{m}_i := \arg\max_{m_i\geq 0} \tilde{v}(m_i) - c_i m_i.
    \end{equation}
    Therefore, all participants making no contribution can not constitute a perfect Bayesian equilibrium.

    Combining the result with that of (\ref{non-positive PBE}), we conclude that no perfect Bayesian equilibrium exists in this case.   
\end{proof}

\subsection{Type Selection}
\begin{tcolorbox}[title=TL;DR]
\textbf{Q:} If the end goal is to \textbf{maximize the coordinator's objective} (i.e., achieving the highest accuracy of the collectively trained model), is it possible that the optimal solution involves disincentivizing certain types of parties from joining the scheme (their IR conditions are not satisfied)? 
\tcblower
\textbf{A:} In general, it is possible. But for the setting of collaborative machine learning we consider in the paper, we can prove that the optimal solution must be such that all types' IR conditions are satisfied.
\end{tcolorbox}
To aid the understanding of the problem in a more general context, we use the conventional term of agents in place parties in the following discussion.

So far we have only considered the case where the IR constraints are satisfied for all types of agents in the population, based on which the principal maximizes the expected model accuracy. In other words, the contract is designed such that all agents find it in their interest to partake in the scheme. Absent from the welfare implication, one legitimate question to ask is whether the principal could be better off if she decides to contract only a subset of the types in the population (i.e., designing a contract that explicitly disincentivize certain types in the population from participation). We call this problem the \textbf{type selection problem}, which is a relatively understudied concept in economics.

We conjecture one of the reasons for the lack of scholarly discussion on this problem is the commonly made assumption of a homogeneous reservation utility, often assumed to be 0 in most cases. This has made the type selection problem trivial, as the null option (demanding 0 contribution and assigning 0 reward), which represents the opt-out option for the agents, is in the design space of the optimal contracting problem with the full set of ordinary IR constraints ($t_i - c_i m_i \geq 0, \forall i \in \{1,\dots, I\}$).

When the reservation utilities of agents are assumed to be homogeneous, denoted by $\tilde{u}$, but not equal to 0, the case becomes slightly involved. The null option (demanding 0 contribution and assigning 0 reward) is no longer in the design space with the full set of ordinary IR constraints ($t_i - c_i m_i \geq \tilde{u}, \forall i$). The options that give agents their reservation utilities (which, in the worse case, requires designing $I$ distinct options, one for each private type) are no longer equivalent to opt-outs, as now these options if chosen by agents will directly affect the principal's objective in a non-trivial manner. Yet, type selection does not complicate the monopolistic screening process by much, due to the transitivity of the IR constraints. To see this, recall that $c_1 > \dots > c_I$. Thus, if $\exists k \in \{1, \dots, I\}, t_k - c_k m_k \geq \tilde{u}$, it follows that $\forall j \geq k, t_k - c_j m_k \geq \tilde{u}$. In words, if the principal wants to contract type-$k$ agents, she must entertain contracting all agents with higher ability, i.e., $j \geq k$. Thus, the type selection problem reduces to determining the lowest type, denoted by $l$, the principal hopes to contract, which can be expressed by a pair of IR constraints:
\begin{equation}
    \left\{\begin{array}{l}
       t_l - c_l m_l \geq \tilde{u} \\
       t_l - c_{l-1} m_l < \tilde{u}
    \end{array}\right. .
\end{equation}
Using backward induction as the solution concept, the screening problem entails solving $I$ optimal contracting problems with different threshold type $l=1,\dots, I$ and then comparing the principal's optimal profits under these $I$ scenarios to determine the best case.

When agents have heterogeneous reservation utilities, the worst case scenario becomes intractable. For a single optimal contracting problem, let $\mathcal{I}_s$ be the set of types a principal selects to contract, and $\mathcal{I} = \{1,\dots, I\}$ all agent types in the population. The IR constraints correspond to the optimal contracting problem are:
\begin{equation}
    \left\{\begin{array}{l}
       t_i - c_i m_i \geq \tilde{u}_i, \forall i\in \mathcal{I}_s\\
       t_l - c_{i} m_l < \tilde{u}_i, \forall l \in \mathcal{I}_s, \forall i\in \mathcal{I} \backslash \mathcal{I}_s.
    \end{array}\right. .
\end{equation}
In the general case, each type may or may not be included in $\mathcal{I}_s$. In total, the principal needs to consider $2^I$ optimal contracting problems before settling down on the best case. This process is detailed in Figure \ref{fig:two-stage problem}.

\begin{figure*}[!htb]
    \centering
    \includegraphics[width=\textwidth]{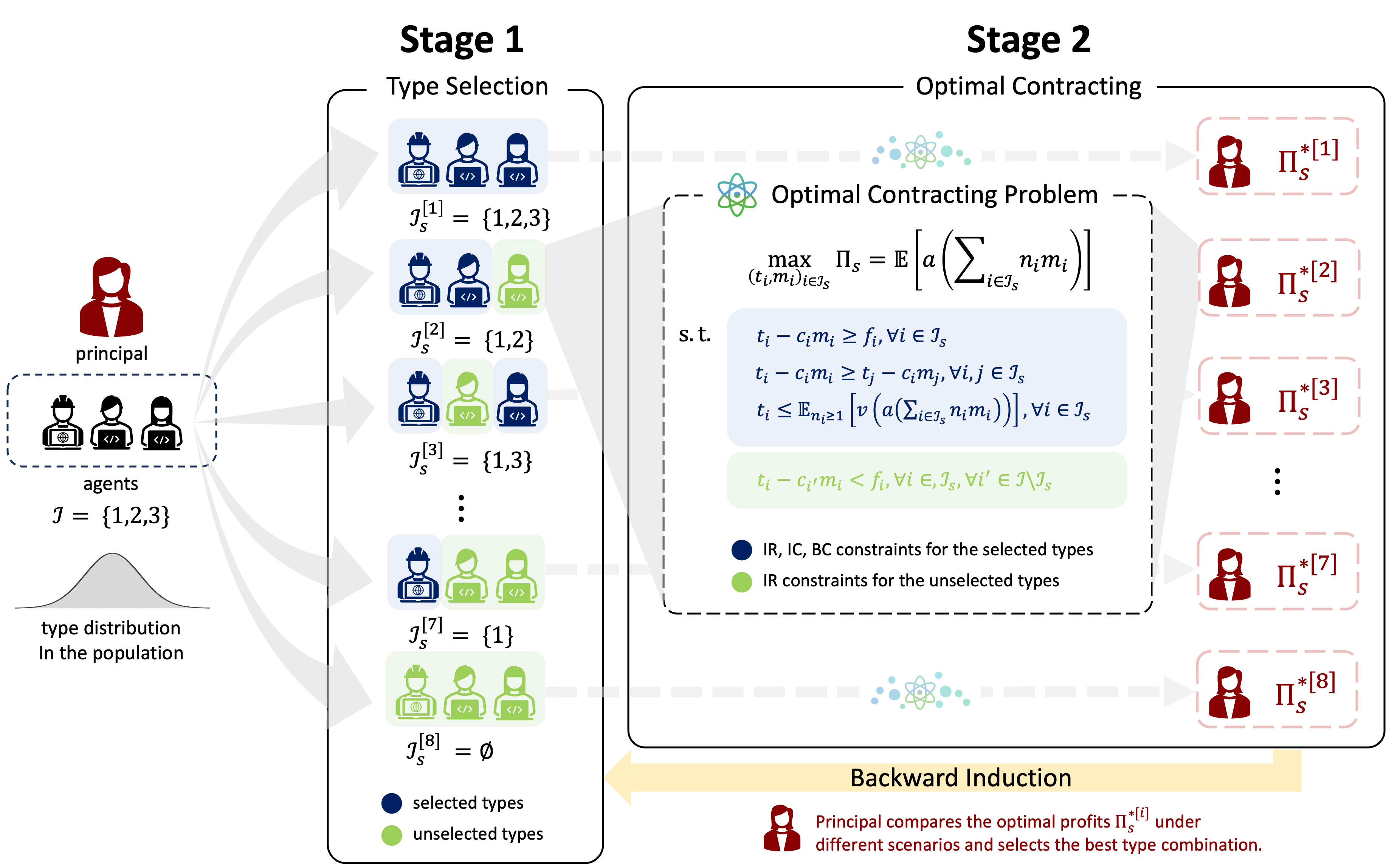}
    \caption{Monopolistic Screening as a Two-Stage Problem.}
    \label{fig:two-stage problem}
\end{figure*}

Returning to our setup of monopolistic screening in collaborative machine learning, we face the situation of heterogeneous reservation utilities. Yet, the situation seems less labyrinthine than the general case discussed above, as the reservation utilities are structurally determined by the individual training problem (\ref{indiv training problem}) and consequently have nice properties, such as monotonicity (i.e., $f_i \leq f_{i+1}$). This motivates us to ask the question of whether type selection in this case can be simplified to the determination of a threshold type, akin to the situation of homogeneous reservation utilities. We show that simplification is indeed possible and moreover the type selection problem becomes altogether \textit{trivial} for monopolistic screening in collaborative machine learning.
\begin{proposition}
    For optimal contract design in collaborative machine learning, the only type selection that the coordinator needs to entertain is the case where all types of parties are selected. Namely, any type selection that entails at least one type of parties not being selected cannot be optimal for the coordinator. Formally, $\forall i \in \mathcal{I}, i \in \mathcal{I}_s$.
\end{proposition}
\begin{proof}
    The key of the proof is to show that based on the optimal solution to the case where a type is not selected we can easily improve upon the optimal solution by appending to the existing contract a designed option that will only be chosen by the unselected type.
    
    Recall the individual training problem:
    \begin{equation}
        \max_{m} \tilde{u} = v(a(m)) - c m
    \end{equation}
    (again dropping the $m\geq 0$ constraint for simplicity), and the FOC
    \begin{equation}
    \tilde{v}^\prime(\bar{m})= c,
    \end{equation}
    where $\tilde{v}(m)\triangleq v(a(m))$. Implicit differentiation yields
    \begin{equation}
        \bar{m}^{\prime}(c) = \frac{1}{\tilde{v}^{\prime\prime}(\bar{m})} <0.
    \end{equation}
    The reservation utility, written as a function of per-unit cost $c$, is defined as the utility level achieved with $\bar{m}$, i.e.,
    \begin{equation}
        f(c) \triangleq \tilde{u}(\bar{m}) = \tilde{v}(\bar{m}) - c \bar{m}.
    \end{equation}
    The derivative of reservation utility with respect to per-unit cost is:
    \begin{align}
    \frac{d f}{d c} = \frac{\tilde{v}^{\prime}(\bar{m})}{\tilde{v}^{\prime\prime}(\bar{m})} - \bar{m} - \frac{c}{\tilde{v}^{\prime\prime}(\bar{m})} = - \bar{m} \leq 0.
    \end{align}

    Define $g(c; i) = t_i - cm_i - f(c)$. We have
    \begin{equation}
        \frac{d g(c; i)}{d c} = \bar{m} - m_i.
    \end{equation}
    or equivalently,
    \begin{equation}
        \frac{d g(c; i)}{d  (-c)} =  m_i - \bar{m}.
    \end{equation}    

    Without loss of generality, let $l\in \mathcal{I}$ be an unselected type in the population, i.e., $l \notin \mathcal{I}_s$. Let $(t_i^*,m_i^*)_{i\in I_s}$ denote the optimal solution to the optimal contracting problem with type selection $\mathcal{I}_s$. By definition,
    $t_i^* - c_l m_i^* < f_l, \forall i \in \mathcal{I}_s$, which means existing contract options will not be selected by parties of type $l$.

    Consider tailoring the option $(t_l, m_l)$ for type-$l$ parties, where $t_l = v(a(\bar{m}_l))$ and $m_l = \bar{m}_l$, i.e., we demand type-$l$ parties contributing the same amount of data as they would utilize when training the model on their own and reward them with the same level of model accuracy as they would obtain with individual model training. This option clearly meets the IR and IC constraints of type-l parties, i.e. $t_l - c_l m_l = f_l \geq t^*_i - c_l m^*_i, \forall i \in \mathcal{I}_s$. The BC constraint holds trivially. The option also does not violate the IC constraints for other selected types. To see this, first note that
    \begin{equation}
        \frac{d g(c_l; l)}{d c} = \frac{d g(c_l; l)}{d (-c)}=0.
    \end{equation}
    For $c > c_l$, 
    \begin{equation}
        \frac{d g(c; l)}{d c} = \bar{m} - \bar{m}_l <0 \label{reserv relation 1}
    \end{equation}
    due to the monotonicity of $\bar{m}$, given by equation (\ref{reserv contribution}). We know that $g(c_l; l)=0$. Consequently, $\forall i \in \mathcal{I}_s$ with $i < l$,
    \begin{equation}
        g(c_i, l) = t_l - c_i m_l - f_i < 0.
    \end{equation}
    Therefore,
        \begin{equation}
         t_l - c_i m_l < f_i \leq t^*_i - c_i m^*_i.
    \end{equation}
    Similarly, for $c < c_l$, 
    \begin{equation}
        \frac{d g(c; l)}{d (-c)} = \bar{m}_l - \bar{m} <0 \label{reserv relation 2}
    \end{equation}
    due to the monotonicity of $\bar{m}$, given by equation (\ref{reserv contribution}). We know that $g(c_l; l)=0$. Consequently, $\forall i \in \mathcal{I}_s$ with $i > l$,
    \begin{equation}
        g(c_i, l) = t_l - c_i m_l - f_i < 0.
    \end{equation}
    Therefore,
        \begin{equation}
         t_l - c_i m_l < f_i \leq t^*_i - c_i m^*_i.
    \end{equation}
    But now with type-$l$ parties contributing non-trivially to the learning scheme and other types of parties contributing the same as before, the optimal solution has been improved.
\end{proof}
It is worth noting that the simplification depends crucially on the common technology, i.e., $f_i$ is determined using the same rule but with varying parameter $c_i$. If we relax this assumption, e.g., allowing parties to determine $f_i$ according to different technologies, of if the coordinator hopes to enact additional fairness requirements through the introduction of $s_i$ (cf. Appendix \ref{enacting additional fairness requirements}), then the type selection problem may no longer be trivial.
\subsection{Convexity of the Optimization Problem}
\begin{tcolorbox}[title = TL;DR]
\textbf{Q:} Is the simplified first-moment problem convex?
\tcblower
\textbf{A:} Yes, we can show the convexity of the problem by showing the convexity of the objective function and constraints.
\end{tcolorbox}

In this section, we prove that the constrained optimization problem under incomplete information is convex. Recall the optimization problem:
\begin{align}
        \max_{\{(t_i, m_i)_{i=1}^{I}\}} & \; \mathbb{E}_{n\sim \mathrm{Multi}(N,p)}\left[a\left(\sum_{i=1}^I n_i m_i\right)\right] \\
    \text{s.t.} & \; \left\{\begin{array}{ll}
       t_i - c_i m_i \geq f_i  &  \forall i; \\
       t_i - c_i m_i \geq t_j - c_i m_j & \forall i,j; \\
       t_i \leq  \mathbb{E}_{n_i \geq 1}\left[v\left(a\left(\sum_{i=1}^I n_i m_i\right)\right)\right], & \forall i.
    \end{array}\right.
\end{align}
Rewrite the problem slightly in standard form and we get
\begin{align}
        \min_{\{(t_i, m_i)_{i=1}^{I}\}} & -\; \mathbb{E}_{n\sim \mathrm{Multi}(N,p)}\left[a\left(\sum_{i=1}^I n_i m_i\right)\right] \\
    \text{s.t.} & \; \left\{\begin{array}{ll}
       f_i - t_i + c_i m_i \leq 0   &  \forall i; \\
        t_j - t_i - c_i (m_j-m_i) \leq 0 & \forall i,j; \\
       t_i -  \mathbb{E}_{n_i \geq 1}\left[v\left(a\left(\sum_{i=1}^I n_i m_i\right)\right)\right] \leq 0, & \forall i.
    \end{array}\right.
\end{align}
Denote the objective function by $\mathcal{O}$ and the left-hand sides of the IR constraints by $\mathcal{IR}_i$, those of the IC constraints by $\mathcal{IC}_{ij}$, and those of the BC constraints by $\mathcal{BC}_i$. The goal is to show that $\mathcal{O}, \mathcal{IR}_i, \mathcal{IC}_{ij}, \mathcal{BC}_i, \forall i,j$ are convex in $(t_i, m_i)_{i=1}^{I}$.
\subsubsection*{Convexity of $\mathcal{O}$}
First, it is obvious that
\begin{equation}
    \frac{\partial \mathcal{O}}{\partial t_i} = 0, \forall i.
\end{equation}
It follows that $\frac{\partial^2 \mathcal{O}}{\partial t_i\partial t_j} = 0, \forall i,j$ and $\frac{\partial^2 \mathcal{O}}{\partial t_i\partial m_j} = 0, \forall i,j$.

Then, consider the derivative with respect to $m_i$.
\begin{equation}
    \frac{\partial \mathcal{O}}{\partial m_i} = - \mathbb{E}\left[n_i a^\prime\left(\sum_{i=1}^I n_i m_i\right)\right], \forall i,
\end{equation}
from which we get the second derivatives:
\begin{equation}
    \frac{\partial^2 \mathcal{O}}{\partial m_i\partial m_j} = -\mathbb{E}\left[n_i n_j a^{\prime\prime}\left(\sum_{i=1}^I n_i m_i\right)\right], \forall i,j.
\end{equation}
We write the Hessian of $\mathcal{O}$ in the block matrix form,
\begin{equation}
    \mathcal{H}_\mathcal{O} = \left[\begin{array}{cc}
        \mathcal{H}_\mathcal{O}^{tt} & \mathcal{H}_\mathcal{O}^{tm} \\
        \mathcal{H}_\mathcal{O}^{mt} & \mathcal{H}_\mathcal{O}^{mm}
    \end{array}\right]
\end{equation}
where the only non-trivial component is $\mathcal{H}_\mathcal{O}^{mm}$ (i.e., $\mathcal{H}_\mathcal{O}^{tt}= \mathcal{H}_\mathcal{O}^{tm}=\mathcal{H}_\mathcal{O}^{mt}=0$). For any $\mathbf{x}\neq 0\in \mathbb{R}^I$,
\begin{equation}
    \mathbf{x}^T \mathcal{H}_\mathcal{O}^{mm} \mathbf{x} = - \mathbb{E}\left[(n_1 x_1 + \dots + n_I x_I)^2  a^{\prime\prime}\left(\sum_{i=1}^I n_i m_i\right)\right]  \geq 0.
\end{equation}
Therefore, $\mathcal{H}_\mathcal{O}^{mm}$ is positive semi-definite. Consequently, $\mathcal{H}_\mathcal{O}$ is positive semi-definite and $\mathcal{O}$ is convex.

\subsubsection*{Convexity of $\mathcal{IR}_i$}
Note that $\mathcal{IR}_i$ is linear in $t_i$ and $m_i$. Therefore, convexity holds trivially.
\subsubsection*{Convexity of $\mathcal{IC}_{ij}$}
Note that $\mathcal{IR}_i$ is linear in $t_i$, $t_j$, $m_i$ and $m_j$. Therefore, convexity holds trivially.
\subsubsection*{Convexity of $\mathcal{BC}_{i}$}
This is proved in a similar fashion to how we prove the convexity of $\mathcal{O}$. Denote the Hessian of $\mathcal{BC}_{i}$ as $\mathcal{H}_{\mathcal{BC}_i}$. Then,
\begin{equation}
    \mathcal{H}_{\mathcal{BC}_i} = \left[\begin{array}{cc}
        0 & 0 \\
        0 & \mathcal{H}_{\mathcal{BC}_i}^{mm}
    \end{array}\right].
\end{equation}
For any $\mathbf{x}\neq 0\in \mathbb{R}^I$,
\begin{equation}
    \mathbf{x}^T \mathcal{H}_{\mathcal{BC}_i}^{mm} \mathbf{x} = - \mathbb{E}_{n_i \geq 1}\left[(n_1 x_1 + \dots + n_I x_I)^2  \tilde{v}^{\prime\prime}\left(\sum_{i=1}^I n_i m_i\right)\right]  \geq 0
\end{equation}
where $\tilde{v}(x)\triangleq v(a(x))$. Therefore, $\mathcal{H}_{\mathcal{BC}_i}^{mm}$ is positive semi-definite. Consequently, $\mathcal{H}_{\mathcal{BC}_i}$ is positive semi-definite and $\mathcal{BC}_i$ is convex. As a result, the constrained optimization problem is convex. Note that the simplified first-moment problem may not be convex, due to the equality of the BC for type-$I$ parties. This issue can be mitigated by re-enacting the inequality of the BC constraint:
\begin{align}
    & \max_{\{(t_i, m_i)_{i=1}^{I}\}} \; \mathbb{E}_{n\sim \mathrm{Mul}(N,p)} \left[ a\left(\sum_{i=1}^I n_i m_i\right)\right] \label{the constrained optimization problem 1} \\
    & \text{s.t.}  \; \left\{\begin{array}{l}
       t_1 - c_1 m_1 - f_1 = 0; \\
       t_i \leq \mathbb{E}_{n_i \geq 1} \left[ v\left( a\left(\sum_{i=1}^I n_i m_i\right)\right)\right], \forall i \in \mathcal{I};\\
       t_i - c_i m_i = t_{i-1} - c_i m_{i-1},  \forall i \in \{2, \dots, I\}; \\
       m_i \geq m_{i-1}, \forall  i \in \{2, \dots, I\};\\
       t_i - c_i m_i - f_i \geq 0,   \forall i \in \{2, \dots, I\}. \\
    \end{array}\right. \label{a convex version of the simplified constrained optimization problem}
\end{align}

\subsection{Existence of Equilibrium}

\begin{tcolorbox}[title = TL;DR]
\textbf{Q:} The \textbf{propositions} in the paper about the properties of optimal contracts \textbf{hinges on the existence of an equilibrium}. Can you show that an equilibrium exists?
\tcblower
\textbf{A:} Yes, we can show the existence of an equilibrium in two steps:
\begin{enumerate}
    \item Prove that the feasible set is not empty.

    \item 
    Show that at least one optimal solution is in a closed and bounded subset of the feasible set.
\end{enumerate}
\end{tcolorbox}
The main results of constraint simplifications in the paper are contingent upon the existence of an equilibrium, which we now show it is indeed the case. We start by showing that the feasible set defined by (\ref{the simplified constrained optimization problem}) is non-empty.
\begin{proposition}
\label{prop: non-empty feasible set}
    There exists at least one candidate solution to monopolistic screening in collaborative machine learning. That is, the principal offers a contract with options stipulating the data usages and model accuracies pertinent to the parties when training a model on their own. Formally, let $\bar{t}_i = v(a(\bar{m}_i))$. The contract $\bar{\mathcal{C}} \triangleq \{\bar{t}_i, \bar{m}_i\}$ satisfies all feasibility constraints under incomplete information.
\end{proposition}
\begin{proof}
    The IR and BC constraints hold trivially. We consider the IC constraints. From (\ref{reserv relation 1}) and (\ref{reserv relation 2}), we know that $\forall l, c_l = \arg\max_c g(c; l)$. Therefore,
    \begin{equation}
        \bar{t}_l - c_i \bar{m}_l < f_i, \forall i \neq l.
    \end{equation}
    Additionally,
    \begin{equation}
        \bar{t}_i - c_i \bar{m}_i = f_i.
    \end{equation}
    Therefore,
    \begin{equation}
        \bar{t}_i - c_i \bar{m}_i \geq \bar{t}_l - c_i \bar{m}_l, \forall i , l\neq i.
    \end{equation}
\end{proof}
Next we show that the feasible set is closed.

\begin{proposition}
\label{prop: closed feasible set}
    The feasible set defined by 
    \begin{equation}
       \left\{\begin{array}{ll}
       t_i - c_i m_i \geq f_i  &  \forall i; \\
       t_i - c_i m_i \geq t_j - c_i m_j & \forall i,j; \\
       t_i \leq  \mathbb{E}_{n_`i \geq 1}\left[v\left(a\left(\sum_{i=1}^I n_i m_i\right)\right)\right], & \forall i.
    \end{array}\right.
    \end{equation}
    is closed.
\end{proposition}
\begin{proof}
    Note that the set represented by each inequality constraint is closed. The intersection of closed sets is also closed.
\end{proof}

Next we show that all $m_i$ and $t_i$ are upper bounded.

\begin{proposition}\label{prop: all mi ti upper bounded}
    All $m_i$ and $t_i$ are upper bounded.
\end{proposition}

\begin{proof}
    As $a(x)\leq 1, \forall x \geq 0$, it follows that $t_i \leq v(1), \forall i$.   Therefore, $t_i$'s are all upper-bounded.
    Re-arranging the IR constraints, we have
    \begin{equation}
        m_i \leq \frac{t_i -f_i}{c_i} \leq \frac{v(1) -f_i}{c_i}.
    \end{equation}
    Hence, all $m_i$ are upper unbounded.
\end{proof}

Now, we use the objective of the optimization problem to show that at least one optimal solution must satisfy the condition that $m_i$ and $t_i$ lowered bounded.

\begin{proposition}
    \label{prop: mi ti lower bounded}
    The optimal solutions must have one that satisfies the condition that $m_i$ and $t_i$ lowered bounded for all $i \in \mathcal{I}$.
\end{proposition}

\begin{proof}
    Suppose there exists a $k\in\mathcal{I}$ such that $m_k = -\infty$. By Proposition \ref{prop: all mi ti upper bounded}, we know that all $m_i$'s are upper bounded. Denote the upper bound by $U$. Now consider the value of the objective function under this scenario:
    \begin{equation}
        \mathbb{E}\left[a\left(n_k m_k + \sum_{i\neq k} n_i m_i\right)\right] \leq  \mathbb{E}\left[a\left(-\infty + \sum_{i\neq k} n_i U\right)\right] = 0.
    \end{equation}
    Since the contract in Proposition \ref{prop: non-empty feasible set} gives a non-negative value of the objective function. We conclude that unless $\bar{m}_i = 0, \forall i \in \mathcal{I}$ s.t. $\mathbb{E}\left[a\left( \sum_{i = 1}^I n_i \bar{m}_i\right)\right] =0$, all optimal solutions, if they exist, must have $m_i$ bounded below. If $\bar{m}_i = 0, \forall i \in \mathcal{I}$, it is possible the the optimal solution yields a value of 0 for the objective function.
    But in this case the contract in Proposition \ref{prop: non-empty feasible set} is indeed an optimal solution that satisfies the condition that all $m_i$'s are lower bounded.

    Now that we have show that there exists an optimal solution where all $m_i$'s are lower bounded. It follows that for this optimal solution, all $t_i$'s are also lower bounded. To see this, denote the lower bound of $m_i$ by $L$.

    Then by the IR constraints, we have
    \begin{equation}
        t_i \geq f_i + c_i m_i \geq f_i + c_i L.
    \end{equation}
    Therefore, all $t_i$'s are lower bounded.
\end{proof}

\begin{proposition}
\label{prop: compact feasible set}
    The modified feasible set defined by 
    \begin{equation}
        \left\{\begin{array}{ll}
       m_i \geq L > -\infty, & \forall i; \\
       t_i - c_i m_i \geq f_i  &  \forall i; \\
       t_i - c_i m_i \geq t_j - c_i m_j & \forall i,j; \\
       t_i \leq  \mathbb{E}_{n_i \geq 1}\left[v\left(a\left(\sum_{i=1}^I n_i m_i\right)\right)\right], & \forall i.
    \end{array}\right.
    \end{equation}
    is closed and bounded, and thus compact, and it is non-empty.
\end{proposition}
\begin{proof}
    The modified feasible set is closed because the intersection of two closed sets is closed. It is bounded due to 
    Propositions \ref{prop: all mi ti upper bounded} and \ref{prop: mi ti lower bounded}. As a result, it is compact.

    The modified feasible set is non-empty, because the contract proposed in Proposition \ref{prop: non-empty feasible set} is contained by it.
\end{proof}

It directly follow from Proposition \ref{prop: compact feasible set} that there exists a solution to the associated optimization problem. Namely, an equilibrium exists.
\subsection{Corner Case}
\begin{tcolorbox}[title = TL;DR]
\textbf{Q:} In proving/stating some of the propositions in the paper, it is assumed that parties have non-zero reservation utilities. How would the analysis change when some parties have \textbf{zero reservation utilities} (i.e., they would not train a model by themselves)?
\tcblower
\textbf{A:} The presence of zero reservation utilities \textbf{would not affect the propositions by much}, with only minor notation changes. All propositions in the main paper and the appendix \textbf{continue to hold}.
\end{tcolorbox}

We have assumed throughout the analysis of our main paper that the reservation utilities of the parties are non-zero, $i.e., \forall i \in \mathcal{I}, c_i \leq \tilde{v}^\prime(0)$. In other words, even the type with the highest per-unit data cost finds it in their interest to utilize some data when training the model on their own. Now we consider the corner case, where some types of parties have 0 reservation utilities. Mathematically, a party of type $i$ will not train a model on their own if
\begin{equation}
    \tilde{v}^\prime(0) < c_i. \label{no training condition}
\end{equation}
This means the marginal cost of contributing one unit of data surpluses the marginal revenue.

In addition to the $I$ types analyzed in the main text, we assume that there are $K$ types that satisfy (\ref{no training condition}). For the ease of notation, we denote the per-unit costs of these types as $c_{-K+1} > \dots > c_0 > \tilde{v}^{\prime}(0)$ and let $\mathcal{K} = \{-K+1, \dots, 0\}$ be the set of indices for these types. By definition, the reservation utilities for the $K$ types are $f_k = 0, \forall k \in \mathcal{K}$. The addition of the $K$ types do not qualitatively change the results of our analysis, as the main propositions in the paper continue to hold with minor changes of notation. It also does not complicate the type selection problem, as the null option $(0,0)$ is in the design place allowing these parties to opt out from the learning scheme. It is worth noting that because these parties have homogeneous reservation utilities which are equal to 0, they together only introduce one additional IR constraint $t_{-K+1} - c_{-K+1} m_{-K+1} \geq 0$—the other IR constraints hold by transitivity. The general optimization problem is:
\begin{align}
        \max_{\{(t_i, m_i)_{i\in \mathcal{I} \cup \mathcal{K}}\}} & \; \mathbb{E}_{n\sim \mathrm{Multi}(N,p)}\left[a\left(\sum_{i\in \mathcal{I} \cup \mathcal{K}} n_i m_i\right)\right] \\
    \text{s.t.} & \; \left\{\begin{array}{ll}
       t_i - c_i m_i \geq f_i  &  \forall i \in \mathcal{I} \cup \mathcal{K}; \\
       t_i - c_i m_i \geq t_j - c_i m_j & \forall i,j \in \mathcal{I} \cup \mathcal{K}; \\
       t_i \leq  \mathbb{E}_{n_i \geq 1}\left[v\left(a\left(\sum_{i\in \mathcal{I} \cup \mathcal{K}} n_i m_i\right)\right)\right], & \forall i \in \mathcal{I} \cup \mathcal{K}.
    \end{array}\right.
\end{align}
The simplified convex first moment problem is:
\begin{align}
     \max_{\{(t_i, m_i \geq 0)_{i\in \mathcal{I} \cup \mathcal{K}}\}} & \; \mathbb{E}_{n\sim \mathrm{Multi}(N,p)}\left[a\left(\sum_{i\in \mathcal{I} \cup \mathcal{K}} n_i m_i\right)\right] \\
    & \text{s.t.}  \; \left\{\begin{array}{l}
       t_{-K+1} - c_{-K+1} m_{-K+1} - f_{-K+1} = 0; \\
       t_i \leq \mathbb{E}_{n_i \geq 1} \left[ v\left( a\left(\sum_{i\in \mathcal{I} \cup \mathcal{K}} n_i m_i\right)\right)\right], \forall i \in \mathcal{I};\\
       t_i - c_i m_i = t_{i-1} - c_i m_{i-1},  \forall i \in \{-K+2, \dots, I\}; \\
       m_i \geq m_{i-1}, \forall  i \in \{-K+2, \dots, I\};\\
       t_i - c_i m_i - f_i \geq 0,   \forall i \in \{1, \dots, I\}. \\
    \end{array}\right. \label{a convex version of the simplified constrained optimization problem new}
\end{align}
\subsection{Graphical Illustration of Disincentivized Contribution}
\begin{tcolorbox}[title = TL;DR]
\textbf{Q:} It seems interesting that that \textbf{parties could contribute less with collaboration} than they would if training a model on their own, which contrasts starkly with mechanisms that incentivize data contribution \cite{karimireddy_mechanisms_2022}. Could you further elucidate the phenomenon?
\tcblower
\textbf{A:} Absolutely, we provide a graphical illustration to clarify the cause of the disincentivized contribution. The core reason for it is the need to enforce \textbf{truth-telling of the higher-ability parties}.
\end{tcolorbox}

Figure \ref{fig:graphical illustration} illustrates the reward-contribution pairs of the optimal contract for Scenario 1 of the multitype case in the experiment section. The thick blue curve represents the model reward as a function of data contribution. The thick red line shows the maximum awardable reward from the collaborative machine learning scheme. Points on the dashed thin lines (termed \textbf{reservation lines}) generate the reservation utilities for the corresponding types of parties. As a result, all admissible contract options for a specific party type must lie on or above its reservation line. From the figure, we can see that the optimal contract option for a type-1 party lies on its reservation line, confirming that the IR condition must bind for the lowest-ability party. All the other types earn utility surpluses from the collaboration, with their options located strictly above the reservation lines. In addition, weak efficiency mandates that the best model be awarded to the highest-ability party. This is corroborated by the figure as the maximum awardable reward line goes through the designed option for Type-5 parties.
\begin{figure}[!htb]
    \centering
    \includegraphics[width=0.8\linewidth]{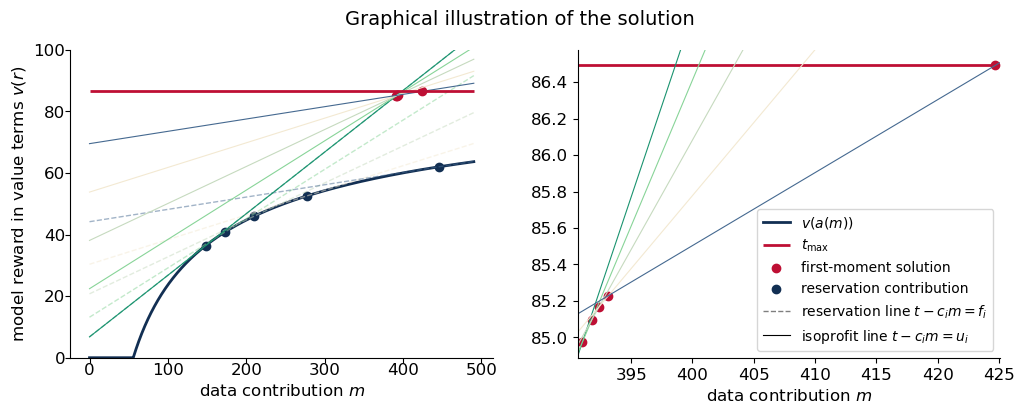}
    \caption{Graphical illustration of disincentivized contribution (Scenario 1 of the multitype case in the experiment section).}
    \label{fig:graphical illustration}
\end{figure}

The solid thin lines are the isoprofit lines for the parties going through the options tailored for them in the optimal contract. Any two points on the isoprofit line are equally preferred by the corresponding party. The truth-telling condition requires that no other types' contract options can lie above a party's isoprofit line that goes through the contract option designed for its type. The right panel of Figure \ref{fig:graphical illustration} reinforces the notion that the downward adjacent comparison constraints must bind for the optimal contracts—each type is indifferent between choosing the option designed for their type, $i$, and that designed for the downward adjacent type $i-1$. This also explains why optimal collaboration can entail reduced contribution from certain types. In the simulated case, requiring a contribution above a type-5 party's reservation level would cause the party to lie about their type and choose the option designed for a type-$4$ party. Whether or not disincentivized contribution is present depends on the curvature of the model reward function (i.e. the training technology), the probability distribution of different types and the cost profiles for these types.

\subsection{Limitations and Future Work}
\begin{tcolorbox}[title = TL;DR]
\textbf{Q:} What are the \textbf{limitations} of the work? How can it be \textbf{improved} in the future?
\tcblower
\textbf{A:} Our goal of this work is to provide a coherent framework—a stepping stone for venturing into the application of contract theory in collaborative machine learning. To improve the implementability and flexibility of the framework, the following areas could prove fertile:
\begin{itemize}
    \item empirical justifications of using model as the reward;
    \item prior-training accuracy measures;
    \item incorporating data quality into the framework;
    \item consideration of additional cost structures;
    \item resolving the combinatorial challenge posed by multinomial distribution.
\end{itemize}
\end{tcolorbox}

In order to effectively qualify the problem and enable tangible theoretical analysis, we have made several restrictive assumptions on the setting. In this section, we discuss the limitation and suggest potential ways to improve upon the paper in the future.
\paragraph{Model as the Reward: Economic Considerations}

In the setup of a party's utility functions, we have assumed the existence of a mapping $v(\cdot)$ from a model's accuracy level to some pecuniary profits. This assumption may seem contrived at first glance. We provide intuition from three perspectives: the analytical perspective, the model market perspective, and the utilitarian perspective.

From an \textbf{\textit{analytical perspective}}, we want to avoid the situation of comparing apples with oranges. In the collaborative machine learning context, this forestalls lumping model rewards/costs and monetary rewards/costs unconverted into a single utility function, which is hard to interpret from a realistic perspective. Some of the previous works try to resolve this issue by the assumption of some weight parameters (cf. \citealp{kang_incentive_2019, ding_incentive_2020}). However, as such weight parameters lack real-world interpretations, justifying the additive format of the utility functions remains a challenge. To ensure comparability and thereby additivity, we could either valuate the model rewards in monetary units, or translate the monetary costs into model training loss. We adopt the first approach here as it is more natural by intuition.

From a \textbf{\textit{model market perspective}}, $v(\cdot)$ can be thought as reflecting the market equilibrium for the trading of the models or services provided by the models. For a specific task on hand, a model with higher accuracy is expected to be sold at a higher price. Here we exclude the possibility of market segmentation, which would permit the same model defined by their accuracy being sold at different prices. This is reasonable as transactions involving models usually happen instantly and digitally and do not experience conventional geographical barriers that hinder the trading of physical goods. One key limitation of the model market perspective is that models are non-exclusive and relatively cheap to replicate. In the absence of legal requirements, the buyer of a model could resell it at a cheap price for multiple times until profits are fully extracted, thereby creating great market volatility. A better perspective is to think about the selling of services provided by the models (\textbf{\textit{model as a service}}). For instance, a participant in the collaborative learning scheme could be a private health institute, and the trained model is used for predictive or advisory purposes. Under this framework, a customer will be willing to pay a higher price for more accurate predictions and better tailored advice, thereby giving rise to the positive correlation between the value of a model and its accuracy level.

Finally from a \textbf{\textit{purely utilitarian perspective}}, models are physical capital that possesses economic values. Be they to supplement medical diagnostics or generate financial consultative advise, they are an indispensable part of the economy in the future of work and translate into increases to the economic output. While measuring the economic surplus generated by models is no mean feat, we conjecture that, ceteris paribus, the higher the quality of a model, the higher its economic value.

Despite the justifications, we acknowledge that the economic landscape of machine learning-driven economy is still burgeoning and stylized facts are yet to be established. In comparison, it may seem more tangible to use money as rewards for the collaboration. The problem setup can be adjusted accordingly, but we expect major conclusions to change as money unlike models cannot be replicated at no cost. One plausible scenario is that the coordinator owns the collectively trained model, use it to generate profits via service provision, and redistribute the profits among the participants/data contributors. The corresponding optimization problem is:
\begin{align}
        \max_{\{(t_i, m_i)_{i=1}^{I}\}} & \; \mathbb{E}_{n\sim \mathrm{Multi}(N,p)}\left[v\left(a\left(\sum_{i=1}^I n_i m_i\right)\right)  - \sum_{i=1}^I n_i t_i \right]\\
    \text{s.t.} & \; \left\{\begin{array}{ll}
       t_i - c_i m_i \geq f_i  &  \forall i; \\
       t_i - c_i m_i \geq t_j - c_i m_j & \forall i,j. \\
    \end{array}\right.
\end{align}

\paragraph{Accuracy function}
One significant determiner of the incentive mechanism is the accuracy function. Challenges pertinent to the choice of the accuracy function come from two front. On the one hand, the accuracy function itself should reflect the performance of the model effectively, and is expected to depend on the model structure, the definition of a validation/test set, the difficulty of the learning task and the quality and quantity of data used for model training. Estimating the accuracy function prior to model training remains an active research area. On the other hand, for practical considerations, participants and the coordinator might have different measures of the model performance, thereby resulting in different accuracy functions. The existing framework could accommodate this generalization, as the first-moment solution solves for the optimal contract in the form of data contributions and model rewards in pecuniary terms. Yet, in this case, the coordinator need to effectively gauge the authenticity of participants' reported accuracy measures to grant suitable models that achieve the values defined by the first-moment solution. This creates a new area of potential information asymmetry and may require the creation of additional contract dimensions for screening purposes.

\paragraph{Independent and identically distributed data assumption}
In the paper, we have made the simplifying assumption of independent and identically distributed data from different participants, which effectively fixes the data quality and makes the accuracy function dependent on the data quantity. In general, the data collected by parties might have different distributions reflecting the geographical or domain differences faced by the parties. We can accommodate this generalization by redefining $m_i$ in the problem setup as the \textbf{effective data size}—a measure of data quantity discounted by the data quality. Practical challenge remains as to find the theoretical support for this conversion. An alternative approach is to explicitly model data quality and data quantity through two separate variables (N.B. effective data size is a special case of this). This results in multi-dimensional contract design, but the key is still on finding and establishing associated theoretical support that have implementation merits.

\paragraph{Cost structure}
In the paper, we assume the per-unit contribution cost to be constant and motivate this assumption with two application scenarios—computing firms pooling GPUs and investment firms combining financial data. In both cases, the per-unit costs remain relatively stable as long as contributions stay within a certain operational scale. Additionally, even if fluctuations do occur, the constant per-unit cost can be interpreted as an average per-unit cost for contributions, making the framework more tangible when resource curation occurs prior to the contracting stage. While the fixed marginal cost assumption covers cases where data collection is standardized, users may want to generalize the framework to cover rising marginal cost. In the latter case, the first-moment setup can still be solved numerically as long as the cost function maintains constraint convexity. Since the cost is per contribution unit, another possible way that allows the framework to generalize is by preprocessing the data into a constant-cost contribution measure (in a similar spirit to the effective data size).

\paragraph{Combinatorial challenge of multinomial distribution}
The assumption of a multinomial distribution for the private types yields a combinatorial set of possible outcomes. This can make numerical optimization extremely exorbitant when the number of types is large in the population. Our preliminary experiments show that the running time of the algorithm increases by roughly a factor of 2 per additional type added, while increments in N have less effect. Consequently, our current model is best suited for cases where $N$ and $I$ are reasonably finite, as is the case of cross-silos collaboration. For reference, it took \textbf{8.03} seconds (wall-clock) to run the experiments in Section 6.2 on a Macbook Pro with M2 chip. Speeding up of the optimization depends crucially on finding a suitable approximation to the objective function and is a potential area for future work.

\end{document}